%% file: neurips_2026.tex
\documentclass{article}

\PassOptionsToPackage{numbers, compress}{natbib}
\usepackage[main, final]{neurips_2026}

\usepackage[utf8]{inputenc} %
\usepackage[T1]{fontenc}    %
\usepackage{hyperref}       %
\usepackage{url}            %
\usepackage{booktabs}       %
\usepackage{amsfonts}       %
\usepackage{nicefrac}       %
\usepackage{microtype}      %
\usepackage[table]{xcolor}  %

\usepackage{microtype}
\usepackage{caption}
\usepackage{graphicx}
\usepackage{subcaption}
\usepackage{booktabs} %
\usepackage{adjustbox}
\usepackage{amsmath}
\usepackage{placeins}
\usepackage{makecell}
\usepackage{booktabs}
\usepackage{multirow}
\usepackage{hyperref}
\usepackage{enumitem}

\definecolor{orange}{HTML}{FFAE7A}
\definecolor{green}{HTML}{A8CF8F}
\definecolor{blue}{HTML}{A9B2E6}
\definecolor{blue1}{HTML}{0018F9}

\definecolor{revise}{HTML}{006884}
\newcommand{\ie}{\textit{i.e., }}
\newcommand{\eg}{\textit{e.g., }}

\newcommand{\oursmetric}{$d_{\text{NTP}}$}
\newcommand{\oursmethod}{\textsl{LTV}}
\setlist[itemize]{
  leftmargin=1em,
  itemsep=1pt,
  topsep=1pt,
  parsep=1pt,
  partopsep=1pt
}

\usepackage{amsmath}
\usepackage{amssymb}
\usepackage{mathtools}
\usepackage{amsthm}
\usepackage{bm}
\usepackage{array}
\usepackage{nccmath}
\usepackage[capitalize,noabbrev]{cleveref}

\theoremstyle{plain}
\newtheorem{theorem}{Theorem}[section]
\newtheorem{proposition}[theorem]{Proposition}

\theoremstyle{definition}

\theoremstyle{remark}

\usepackage[textsize=tiny]{todonotes}

\newcommand\revise[1]{#1}
\usepackage[normalem]{ulem}

\AtBeginDocument{
  \setlength{\abovedisplayskip}{4pt plus 2pt minus 2pt}
  \setlength{\belowdisplayskip}{4pt plus 2pt minus 2pt}
  \setlength{\abovedisplayshortskip}{4pt plus 2pt minus 2pt}
  \setlength{\belowdisplayshortskip}{4pt plus 2pt minus 2pt}
}
\input{math_commands}
\title{Distributional Alignment as a Criterion for \\ Designing Task Vectors in In-Context Learning}

\author{%
  Jihoon Kwon\thanks{Equal Contribution $^\dagger$Corresponding author} \\
  Seoul National University \\
  \texttt{kog0712@snu.ac.kr}
  \And
  Jiwon Choi$^*$ \\
  Yonsei University \\
  \texttt{jiii111@yonsei.ac.kr}
  \And
  Jy-yong Sohn$^{\dagger}$ \\
  Yonsei University \\
  \texttt{jysohn1108@yonsei.ac.kr}
}

\begin{document}

\maketitle

\input{main}

\newpage
\input{appendix}

\newpage
\input{checklist.tex}

\end{document}

%% file: math_commands.tex
\def\eqref#1{equation~\ref{#1}}

\def\1{\bm{1}}

\def\vc{{\bm{c}}}

\def\vh{{\bm{h}}}

\def\vp{{\bm{p}}}
\def\vq{{\bm{q}}}

\def\vv{{\bm{v}}}
\def\vw{{\bm{w}}}
\def\vx{{\bm{x}}}
\def\vy{{\bm{y}}}
\def\vz{{\bm{z}}}

\def\mH{{\bm{H}}}
\def\mI{{\bm{I}}}

\def\mV{{\bm{V}}}
\def\mW{{\bm{W}}}

\def\mY{{\bm{Y}}}

\DeclareMathAlphabet{\mathsfit}{\encodingdefault}{\sfdefault}{m}{sl}
\SetMathAlphabet{\mathsfit}{bold}{\encodingdefault}{\sfdefault}{bx}{n}

\newcommand{\softmax}{\mathrm{softmax}}

\newcommand{\KL}{D_{\mathrm{KL}}}

%% file: main.tex
\begin{abstract}
In-context learning (ICL) allows large language models (LLMs) to adapt to new tasks through demonstrations, yet it suffers from escalating inference costs as context length increases.
While task vectors offer a promising alternative by compressing demonstrations into compact hidden-state representations, their quality has been evaluated only through downstream task accuracy.
This indirect criterion provides limited insight into how to design more effective task vector extraction methods.
In this paper, we posit that inference using task vectors should align their predictive distribution with that of ICL.
To quantify this, we introduce $d_{\text{NTP}}$, a metric that measures the discrepancy in next-token probabilities between task vector-based and ICL-based inference.
Our empirical analysis reveals that $d_{\text{NTP}}$ serves as a performance proxy, exhibiting a strong negative correlation with downstream accuracy.
Motivated by this, we develop Linear Task Vector (LTV), a method designed to minimize $d_{\text{NTP}}$ via a closed-form linear mapping that estimates demonstration effects through regression.
Across eight classification benchmarks and five LLMs, LTV consistently outperforms existing task vector baselines, improving average accuracy by 9.2\% while reducing inference latency.
We further show that LTV outperforms the baselines on regression tasks.
Moreover, we investigate the transferability of LTV across different model scales; an aspect that has remained nascent in task vector research. Specifically, we empirically show that task vectors from a larger model can enhance a smaller model's performance by 6.4\%, suggesting a new utility for extracted task representations.
\end{abstract}

\section{Introduction}
\label{sec:intro}

In-Context Learning (ICL) has emerged as a powerful paradigm for adapting large language models (LLMs) to new tasks by simply prepending labeled demonstrations before the query~\citep{brown2020language, wei2022emergent}.
ICL has been shown to achieve impressive performance gain across various tasks without requiring any model parameter updates, with performance typically improving as more demonstrations are provided~\citep{agarwal2024many, dong2024survey}.
However, this improvement comes at a cost: longer demonstrations either increase inference-time computation as the input length grows, or incur memory overhead when caching the activations~\citep{mu2023learning, li2024implicit, gao2025optimization, golchintowards}.
These computational and memory limitations hinder the practical utility of ICL under resource constraints.

To address these limitations, recent work has proposed \emph{task vectors in ICL} as a training-free approach that enables task adaptation without directly using demonstrations at inference time~\citep{hendel2023context, todd2023function, liu2023context, li2024context}.
In this approach, a task vector (TV) refers to a condensed vector extracted from the internal activations of an LLM when performing ICL, encapsulating the task information that the LLM implicitly infers from demonstrations~\citep{han2024emergence, yang2025task, dong2025understanding}.
By applying this vector to zero-shot inference, where no demonstrations are provided, the model achieves substantial performance gains on the task~\citep{hendel2023context}.
While various TV extraction methods have been proposed in the past years, the downstream task performance remains the only established criterion 
for comparing them, offering limited insight into \emph{why} 
one method outperforms another and \emph{how} to improve the existing extraction methods.

In this paper, we posit that distributional alignment with ICL is a desirable property of task vectors, and a useful criterion for designing TV extraction methods.
This perspective is motivated by the following idea: since the role of a task vector is to condense the effect of demonstrations, TV-based inference should naturally produce a predictive distribution that closely aligns with that of ICL.

From this perspective, we make the following contributions:
\begin{itemize}
    \item We propose \oursmetric{}, a metric that quantifies the quality of TV methods, by measuring the \emph{discrepancy} between the predictive distribution under TV-based inference and that under ICL-based inference, in terms of \emph{next-token probability (NTP)}.
    We empirically show that \oursmetric{} has a strong negative correlation with the downstream performance, serving as an indicator of the quality of TV.
    \item We develop the Linear Task Vector (\oursmethod{}) method, which is designed to reduce \oursmetric{}. 
    Specifically, \oursmethod{} employs a linear mapping that estimates the effect of demonstrations, and uses the closed-form solution of a regression problem to extract task vectors.
    \item In experiments, \oursmethod{} consistently outperforms existing TV methods on eight classification benchmarks and five LLMs, improving average accuracy by 9.2\% while reducing inference latency.
    Furthermore, \oursmethod{} outperforms the baselines on regression tasks.
    Finally, we extend our study to the transferability of task vectors, a dimension largely unexplored in existing research. By applying task vectors extracted from a larger model to smaller ones -- which may have limited capacity or context lengths -- we achieve a 6.4\% improvement in classification accuracy.
\end{itemize}

\section{Related Work}

\paragraph{Task Vectors.}
A pioneering work~\citet{ilharco2022editing} introduces the concept of task vectors, defined as the difference in parameter space between a pre-trained model and the model fine-tuned for a specific task.
The core idea is that shifting the weights of a model in the task vector direction improves performance on that task~\citep{ortiz2023task, zhang2024knowledge, li2025task}.
Recent works have shown that task vectors can also be extracted from representation spaces, such as activation spaces~\citep{hendel2023context, yang2025task} or soft prompt spaces~\citep{belanec2025task}.

\paragraph{In-Context Learning.}
ICL enables LLMs to adapt to new tasks by simply prepending query-label pairs as demonstrations to the model input~\citep{brown2020language}.
The success of ICL has motivated various theoretical interpretations~\citep{zhou2024mystery, von2023transformers, ahn2023transformers, olsson2022context,yin2025attention, akyurek2022learning, li2023transformers}.
A notable line of work interprets ICL as implicit Bayesian inference~\citep{xie2021explanation, panwar2023context, zhang2023and}: as the model processes demonstrations, it implicitly infers a \emph{latent task concept} from demonstrations and conditions its predictions on the resulting posterior.
This view provides the theoretical foundation for task vectors in ICL~\citep{mittaldoes}, a line of work we describe in detail below.

\paragraph{Task Vectors in ICL.}
A key limitation of ICL is that it incurs substantial inference-time computation and memory overhead as the input length increases~\citep{li2024implicit, agarwal2024many}.
To reduce these overheads, recent works aim to internalize the task adaptation induced by ICL, by adjusting the model parameters or activations so that the effects of demonstrations are encoded within the model itself.
One line of work achieves this through few-shot parameter-efficient fine-tuning (PEFT)~\citep{li2024implicit, jukic2024disentangling, gao2025optimization, li2025towards}.
Another line of work explores task vectors in ICL, offering a training-free alternative that enables task adaptation.

Numerous methods have been proposed for extracting task vectors in ICL, demonstrating performance gains over zero-shot inference~\citep{hendel2023context, todd2023function, liu2023context, li2024context, li2024implicit, kang2025adaptive, wang2024elicit}.
As the model activations span multiple modules, existing methods vary widely in where and how task vectors are extracted.
This diversity underscores the need for a direct criterion to evaluate the quality of extracted task vectors.

\section{Backgrounds}
\label{sec:backgrounds}

In this section, we first review relevant concepts and notations used in our paper. 
Sec.~\ref{sec:model} describes how LLMs predict the next token, Sec.~\ref{sec:task} defines our target classification task, and Sec.~\ref{sec:mode} introduces three inference modes for LLMs -- zero-shot, ICL, and using task vectors.

\subsection{Model: Large Language Models (LLMs)}
\label{sec:model}

We consider pre-trained auto-regressive LLMs which predict the next token $u$ given an input prompt $p$, a sequence of tokens.
The model consists of three components:
\begin{itemize}
    \item the embedding layer that converts each token in the prompt $p$ into an embedding vector,
    \item the transformer~\citep{vaswani2017attention} (TF) decoder consisting of $L$ layers, which transforms the sequence of embedding vectors into a sequence of hidden states,
    \item the language modeling (LM) head that predicts the probability of the next token $u$ based on the output of TF.
\end{itemize}

Let $[\vh_1(p), \vh_2(p), \dots, \vh_l(p)]$ denote the hidden states at the final layer of the TF decoder, where $l$ is the sequence length and $\vh_l(p) \in \mathbb{R}^d$ is a $d$-dimensional vector.
The LM head predicts the next token $u \in \mathcal{U}$ by applying a linear projection to the last hidden state $\vh_l(p)$, where $\mathcal{U} = \{1, 2, \ldots, N_{\mathcal{U}}\}$ is the vocabulary set; each token is represented by its index.
To be specific, the probability of the next token is computed as:
\begin{equation}
    P(u \mid p) = \sigma(\mW_{\mathrm{lm}} \vh_l(p))[u], \quad u \in \mathcal{U},
\end{equation}
where $\mW_{\mathrm{lm}} \in \mathbb{R}^{N_\mathcal{U} \times d}$ denotes the weight matrix of the LM head, and $\sigma(\cdot)$ denotes the softmax function.
For notational simplicity, we hereafter write $\vh(p)$ for $\vh_l(p)$, as we only use the hidden state of the last token.
We also write $\text{TF}(p)$ to denote the hidden state $\vh(p)$ obtained by embedding the prompt $p$ and passing it through TF.

\subsection{Task: Classification}
\label{sec:task}

While large language models (LLMs) can be applied to a wide range of downstream tasks, prior work on task vectors in ICL~\citep{hendel2023context, li2024implicit, saglam2025learning} has primarily focused on classification settings.
Following this line of work, we also restrict our attention to classification tasks.

We define a classification task by a distribution $\mathcal{D}$ over query--label pairs $(x, y)$, where the query $x$ is a text sequence and the label $y$ belongs to a task-specific label set $\mathcal{C} \subseteq \mathcal{U}$ which contains $\lvert \mathcal{C} \rvert = K$ classes.
Given a query $x$, the goal is to predict its corresponding label $y$.
We consider the next-token distribution \emph{restricted} to the label set $\mathcal{C}$:
\begin{equation}
\label{eq:restricted_prob}
    P(c \mid p; \mathcal{C}) = \mfrac{P(c \mid p)}{\sum_{c^\prime \in \mathcal{C}} P(c^\prime \mid p)}, \quad c \in \mathcal{C}.
\end{equation}
For notational simplicity, we hereafter write $P(c \mid p)$ to denote this restricted distribution.
In greedy decoding, the predicted label $\hat{y}$ is determined by selecting the class with the highest probability:
\begin{equation}
    \hat{y} = \mathrm{argmax}_{c \in \mathcal{C}} \, P(c \mid p).
\end{equation}

\input{figure/modes}

\subsection{Methods: Inference Modes for LLMs}
\label{sec:mode}

We introduce three inference modes for LLMs: zero-shot inference in Sec.~\ref{subsubsec:zero}, in-context learning (ICL) in Sec.~\ref{subsubsec:icl}, and task vectors in Sec.~\ref{subsubsec:tv}.

\subsubsection{Zero-shot Inference Mode}
\label{subsubsec:zero}

In the zero-shot inference mode, the LLM predicts $y_{\text{test}}$ for the test query $x_{\text{test}}$, without being provided any labeled examples $(x,y)$ for the target task. 
The leftmost part of Fig.~\ref{fig:modes} shows the detailed process.
First, the test query $x_{\text{test}}$ is passed through the TF to obtain the hidden state $\vh_{\text{zs}} = \text{TF}(x_{\text{test}})$.
Then, the LM head computes the probability $P(c \mid x_{\text{test}})$ for each class $c$ from $\vh_{\text{zs}}$.
Finally, the predicted label is selected via greedy decoding:
\begin{equation}
\label{eq:zs}
    \hat{y}_{\text{zs}} = \mathrm{argmax}_{c \in \mathcal{C}} \, P(c \mid x_{\text{test}}).
\end{equation}
Throughout the paper, we use the subscript `zs' to note that the quantity is for the zero-shot inference.

\subsubsection{In-Context Learning (ICL) Mode}
\label{subsubsec:icl}
Suppose we are given $k$ demonstrations $Z = \{(x_i, y_i)\}_{i=1}^k$ sampled from the task distribution $\mathcal{D}$.
As shown in the middle of Fig.~\ref{fig:modes}, ICL prepends $Z$ to the test query $x_{\text{test}}$ and passes them through the model, which outputs $\vh_{\text{icl}} = \text{TF}([Z|| x_{\text{test}}])$, where $||$ represents the concatenation operator.
The LM head then computes the probability $P(c \mid [Z|| x_{\text{test}}])$ for each class $c$ from $\vh_{\text{icl}}$.
We denote this probability by $P_{\text{icl}}(c \mid x_{\text{test}}, Z)$ to indicate the ICL inference mode.
Lastly, the predicted label is then obtained as
\begin{equation}
\label{eq:icl}
    \hat{y}_{\text{icl}} = \mathrm{argmax}_{c \in \mathcal{C}} P_{\text{icl}}(c \mid x_{\text{test}}, Z).
\end{equation}

\subsubsection{Task Vector (TV) Mode}
\label{subsubsec:tv}

Inference using task vectors is a variant of ICL.
This mode is motivated by the implicit Bayesian view~\citep{xie2021explanation}, which posits that during ICL, the LLM predicts the label $y$ conditioned on a latent task concept $\vv$ inferred from demonstrations $Z$.
Under this view, the predictive distribution computed by the LLM under the ICL can be expressed as
\begin{equation}
\label{eq:bayesian}
    P(y \mid x, Z) = \int_\vv P(y \mid x, \vv) \, P(\vv \mid Z) \, d\vv.
\end{equation}
where $P(\vv \mid Z)$ denotes the posterior over the latent task concept $\vv$ inferred from the $Z$, and $P(y \mid x, \vv)$ denotes the predictive distribution of the label $y$ conditioned on the inferred concept $\vv$.

TV mode explicitly makes use of the decomposition of the predictive distribution as in the right-hand side of~\eqref{eq:bayesian}.
In other words, the task vector mode is composed of two stages: (1) the \emph{extraction} stage, which extracts a task vector $\vv$ from the model activation induced by the demonstrations $Z$, and (2) the \emph{inference} stage, which applies the extracted vector $\vv$ to the model activation, in the zero-shot inference.
Below we formally describe each stage, which is shown in the rightmost column of Fig~\ref{fig:modes}.

\paragraph{Extraction of task vector.} 
Let $f$ denote a task vector extraction function that takes demonstrations $Z$ and query $x$ as input, and outputs a task vector $\vv$.
Formally, we represent the task vector $\vv$ as
\begin{equation}
\label{eq:tv-general}
    \vv = f(x, Z).
\end{equation}
For notational simplicity, we suppress the dependence of $\vv$ on $x$ and $Z$ when it is clear from context.

\paragraph{Inference using task vector.} 

Recall that in the zero-shot inference mode (specified in Sec.~\ref{subsubsec:zero}), TF outputs the hidden state of the last token $\vh_{\text{zs}}$, when the input is set to the test query $x_{\text{test}}$.
In the TV mode, the task vector $\vv$ extracted in~\eqref{eq:tv-general} is used to update\footnote{While our method additively updates the output of TF, there exist other task vector-based methods that combine a model activation and $\vv$ in different ways. We focus on this additive case for notational simplicity.} the hidden state from $\vh_{\text{zs}}$ to task-conditioned hidden state $\vh_{\text{tv}} = \vh_{\text{zs}}+\vv$.
For a given $\vh_{\text{tv}}$, the LM head computes the probability $P(c \mid x_{\text{test}}, \vv)$ for each class $c$; we denote this probability by $P_{\text{tv}}(c \mid x_{\text{test}}, \vv)$ to indicate that this inference mode uses task vectors.
The predicted label is then determined as
\begin{equation}
\label{eq:tv}
    \hat{y}_{\text{tv}} = \mathrm{argmax}_{c \in \mathcal{C}} P_{\text{tv}}(c \mid x_{\text{test}}, \vv).
\end{equation}

\input{figure/metric}

\section{Measuring the Quality of Task Vector}
\label{sec:metric}

We propose a metric that measures the quality of the task vector extracted by $f$.
We first provide the motivation of our approach in Sec.~\ref{subsec:motiv}, and then propose our metric in Sec.~\ref{subsec:def_metric}.
Finally, we show that our metric serves as an indicator of task vector quality in Sec.~\ref{subsec:met_exp}.

\subsection{Motivation}
\label{subsec:motiv}

Recall that the goal of using task vectors is to condense the task 
information inferred during ICL into a compact vector $\vv$, enabling 
predictions similar to ICL without the overhead of processing 
demonstrations.
However, prior work has relied solely on downstream task accuracy 
to evaluate the quality of the task vector.
This evaluation practice offers limited insight into \emph{why} 
one method outperforms another, and provides little guidance on 
\emph{how} to design better task vector methods.

To address this, we propose a metric grounded in the goal of using task vectors: \emph{enabling predictions similar to ICL}.
This goal means that, in the formulation of Sec.~\ref{sec:mode}, adding $\vv$ to $\vh_{\text{zs}}$ should approximate the effect of prepending demonstrations $Z$ to the query $x$.
\revise{The resulting distribution $P_{\text{tv}}$ in~\eqref{eq:tv} is thus intended to approximate $P_{\text{icl}}$ in~\eqref{eq:icl}.}
Therefore, if the task vector successfully captures the task information inferred during ICL, its inference should yield a predictive distribution \emph{aligned} with that of ICL.

We thus suggest to measure the discrepancy between the two distributions, formally defined as below.

\subsection{Proposed Metric}
\label{subsec:def_metric}

Recall that $P_{\text{icl}}$  and $P_{\text{tv}}$ are the probabilities of the next token computed by the model, for the ICL mode and the TV mode, respectively. 
Given demonstrations $Z$, we measure the quality of the task vector extraction method $f$ as the discrepancy of the ICL mode and the TV mode in terms of the next token probability (NTP), denoted by
\begin{equation}
\label{eq:metric}
\text{\oursmetric}(f; Z) = \mathbb{E}_{x\sim\mathcal{D}} \Big[
    \KL\big(
        P_{\text{icl}}(\cdot \mid x, Z)
        \;\|\;
        P_{\text{tv}}(\cdot \mid x, f(Z))
    \big)
\Big],
\end{equation}
where $\KL$ represents the Kullback-Leibler (KL) divergence~\citep{kullback1951information} operator. 
Here, $P_{\text{icl}}$ serves as the reference distribution, and the metric quantifies how far $P_{\text{tv}}$ deviates from the reference.
See Fig.~\ref{fig:metric} for the illustration of our proposed metric.

A lower \oursmetric{} indicates that $P_{\text{tv}}$ is more closely aligned with $P_{\text{icl}}$.
Thus, the task vector extraction method $f$ with lower \oursmetric$(f)$ is considered a more desirable method. Notably, our proposed metric evaluates the quality of task vector extraction methods without requiring labels for the test set.

\input{figure/fig_corr}

\subsection{Correlation of Proposed Metric With Test Accuracy}
\label{subsec:met_exp}

Now, a valid question is, whether our proposed metric \oursmetric$(f)$ in~\eqref{eq:metric} is a good indicator for the quality of the task vector extraction method $f$, in the practical sense.
In this section, we empirically validate that \oursmetric$(f)$ strongly correlates with the test accuracy when task vector $\vv = f(Z)$ is used, across diverse classification benchmarks. 

\paragraph{Experimental Setup.}
We evaluate on four classification benchmarks -- AGNews, DBPedia~\citep{zhang2015character}, MR~\citep{pang2004sentimental}, SST-2~\citep{socher2013recursive}.
We compute the metric \oursmetric$(f)$ and the test accuracy for four task vector extraction methods $f$ -- Function Vector~\citep{todd2023function}, Task Vector~\citep{hendel2023context}, State Vector~\citep{li2024context}, and I2CL~\citep{li2024implicit}. We test on LLaMA-3.1-8B~\citep{dubey2024llama} and Qwen-2.5-7B~\citep{qwen2024qwen2} models using 30 demonstrations $Z$, across 20 independent runs with randomly sampled demonstrations $Z$.
See Appendix~\ref{app:exp_details} for further details.

\paragraph{Results.}
Fig.~\ref{fig:fig_corr} presents scatter plots of \oursmetric$(f)$ versus accuracy $\text{Acc}(f)$, where each point corresponds to each result obtained with different demonstrations $Z$.
Across all tasks and models, lower \oursmetric$(f)$ consistently correlates with higher accuracy, with most Pearson correlation coefficients exceeding $0.6$ in magnitude.
This strong correlation demonstrates that lower \oursmetric$(f)$ indicates superior task vector quality.
This result motivates developing a task vector extraction method aimed at reducing \oursmetric$(f)$, as proposed in the next section.

\section{Proposed Method: Linear Task Vector}
\label{sec:method}

In Sec.~\ref{sec:metric}, we empirically observed that in the TV mode, the test accuracy negatively correlates with the discrepancy \oursmetric$(f)$ measured for the task vector extraction method $f$.
This motivates us to devise a task vector extraction method that achieves a small \oursmetric$(f)$, which is expected to improve test accuracy when the task vector is used.

Based on this motivation, we propose Linear Task Vector (\oursmethod{}), which leverages a linear 
mapping to compute task vectors that enable TV mode inference 
to closely resemble that of ICL.
We first present the rationale behind our approach in Sec.~\ref{subsec:ltv_rationale},
then formally describe \oursmethod{} in Sec.~\ref{subsec:ltv}.

\subsection{Rationale Behind the Proposed Method}
\label{subsec:ltv_rationale}

The proposed \oursmethod{} method is designed to achieve the following two goals simultaneously.
First, we aim to design a method $f$ that achieves a small \oursmetric$(f)$, meaning that the next token probability $P_{\text{tv}}$ under TV mode closely replicates $P_{\text{icl}}$ under ICL mode.
Second, we seek to preserve the key advantage of ICL, which does not require updating the parameters of a model.
These goals lead us to formulate a proxy optimization problem which satisfies two conditions: (1) the proxy objective is provably related to the original objective \oursmetric{}, and (2) the proxy problem has a closed-form solution.

With these goals in mind, we briefly explain the intuition behind our approach before formulating the optimization problem.
Note that the only difference in computing $P_{\text{icl}}$ and $P_{\text{tv}}$ is the input hidden state, $\vh_{\text{icl}}$ and $\vh_{\text{tv}}$, respectively.
We therefore expect that making $\vh_{\text{tv}}$ close to $\vh_{\text{icl}}$ will make $P_{\text{tv}}$ close to $P_{\text{icl}}$.
This motivates finding $f$ that extracts a task vector $\vv$ to bring $\vh_{\text{tv}}$ close to $\vh_{\text{icl}}$ for a given test query $x_{\text{test}}$.
Since $\vh_{\text{tv}}=\vh_{\text{zs}}+\vv$, this amounts to estimating $\vv \approx \vh_{\text{icl}}-\vh_{\text{zs}}$.

\input{figure/method}

\paragraph{Formulation of the proxy optimization problem.}
Given demonstrations $Z$, our goal is to find the optimal $f$ that minimizes the discrepancy \oursmetric$(f)$ between $P_{\text{icl}}$ and $P_{\text{tv}}$.
As depicted in Fig.~\ref{fig:metric}, the predictive distributions $P_{\text{icl}}$ and $P_{\text{tv}}$ are obtained by applying the LM head to the respective hidden states $\vh_{\text{icl}}$ and $\vh_{\text{tv}}$.
Thus, the metric \oursmetric$(f)$ in~\eqref{eq:metric} can be expressed as:
\begin{equation}
\nonumber
    \text{\oursmetric{}}(f) = 
    \mathbb{E} \Big[
    \KL
    \big(
    \sigma(\mW_{\mathrm{lm}} \vh_{\text{icl}})\;\|\;
    \sigma(\mW_{\mathrm{lm}}(\vh_{\text{zs}} + f(Z)))
    \big)
    \Big],
\end{equation}
where $\vh_{\text{tv}} = \vh_{\text{zs}} + \vv = \vh_{\text{zs}} + f(Z)$ is the hidden state of TV mode.
Thus, minimizing \oursmetric$(f)$ with respect to $f$ leads to the following optimization problem:
\begin{equation}
\label{eq:opt_kl}
    \min_{f} \;
    \mathbb{E} \Big[
    \KL\big(
    \sigma(\mW_{\mathrm{lm}} \vh_{\text{icl}})
    \;\|\;
    \sigma(\mW_{\mathrm{lm}}(\vh_{\text{zs}} + f(Z)))
    \big)
    \Big].
\end{equation}
However, solving~\eqref{eq:opt_kl} requires iteratively updating the task vector extraction model $f$, as no closed-form solution exists.
In order to preserve the key advantage of ICL (which does not require any model updates), we instead introduce a proxy objective in a way that the corresponding minimization problem has a closed-form solution on $f$.
To this end, we define the proxy objective as the mean squared error (MSE) between the hidden states $\vh_{\text{icl}}$ and $\vh_{\text{tv}}$:
\begin{equation}
\label{eq:mse}
    \mathcal{L}_{\text{MSE}}(f) = \mathbb{E} \Big[ \big\| \vh_{\text{icl}} - \vh_{\text{tv}} \big\|_2^2 \Big] = \mathbb{E} \Big[ \big\| \vh_{\text{icl}} - \vh_{\text{zs}} - f(Z) \big\|_2^2 \Big].
\end{equation}
Recall that \oursmetric{} in~\eqref{eq:metric} is defined as the expectation over query $x$.
With explicit dependence on $x$ and demonstrations $Z$, the proxy optimization problem becomes:
\begin{equation}
\label{eq:opt_mse}
    \min_f \mathbb{E}_{x\sim \mathcal{D}} \Big[ \big\| \vh_{\text{icl}}(x, Z) - \vh_{\text{zs}}(x) - f(x,Z) \big\|_2^2 \Big].
\end{equation}

A natural question is how this proxy objective $\mathcal{L}_{\text{MSE}}$ is related to the original objective \oursmetric{}.
The following proposition answers this question: 

\begin{proposition}[Relationship between \oursmetric{} and $\mathcal{L}_{\text{MSE}}$]
\label{prop:kl_bound}
Assume the language modeling head $\mW_{\text{lm}}$ has a bounded spectral norm $\|\mW_{\text{lm}}\|_2 \le C_1$ and the log-softmax function is $C_2$-Lipschitz in the $\ell_2$-norm.
Then, for any function $f$
\begin{equation}
    \text{\oursmetric{}}(f) \le C_1 C_2 \sqrt{\mathcal{L}_{\text{MSE}}(f)}.
\end{equation}
\end{proposition}

We defer the proof of Proposition~\ref{prop:kl_bound} to Appendix~\ref{app:proof_kl_bound}.
This proposition indicates that we can reduce \oursmetric{}$(f)$ by finding $f$ which has small $\mathcal{L}_{\text{MSE}}(f)$.

\paragraph{Using a linear mapping to solve the proxy problem.}
We now focus on 
solving the proxy problem in~\eqref{eq:opt_mse}, \ie finding $f$ which outputs the task vector $\vv=f(x,Z)$ that resembles $\vh_{\text{icl}} - \vh_{\text{zs}}$.
In other words, the purpose of task vector $\vv$ is to compensate for the effect of demonstrations in the ICL mode (compared with the zero-shot inference mode which does not use demonstrations) in the hidden state. Thus, a natural approach to solve this problem is to find a mapping from $\vh_{\text{zs}}$ (the hidden state for the zero-shot inference mode) to the target $\vh_{\text{icl}} - \vh_{\text{zs}}$.
In order to get a closed-form solution for the optimization problem, we consider the simplest case\revise{\footnote{In Appendix~\ref{app:design_choices}, we empirically compare this linear choice with alternatives, and discuss non-linear extensions.}} of using a \emph{linear} mapping $\mW$, which results in the following formulation: 
\begin{align*}
    f(x,Z) = \mW(Z) \vh_{\text{zs}}(x).
\end{align*}
This leads to the following optimization problem:
\begin{equation}
\label{eq:lin_opt}
    \min_{\mW} \mathbb{E}_{x\sim \mathcal{D}} \Big[ \big\| \vh_{\text{icl}}(x, Z) - \vh_{\text{zs}}(x) - \mW \vh_{\text{zs}}(x) \big\|_2^2 \Big].
\end{equation}
This reformulated problem in~\eqref{eq:lin_opt} has a closed-form solution, as described below.

\subsection{Formal Description of Proposed Method}
\label{subsec:ltv}

We now formalize Linear Task Vector (\oursmethod{}), a task vector extraction method $f$ which estimates a linear mapping $\mW^{\star}$ that minimizes~\eqref{eq:lin_opt}.
Specifically, in the extraction stage, \oursmethod{} estimates $\mW^{\star}$ via ridge regression~\citep{hoerl1970ridge}.
In the inference stage, the estimated optimal mapping $\mW^{\star}$ is applied to compute the task vector.

\paragraph{Extraction Stage.}
Given demonstrations $Z$, we estimate $\mW \in \mathbb{R}^{d \times d}$ by solving a regression problem mapping $\vh_{\text{zs}}$ to $(\vh_{\text{icl}} - \vh_{\text{zs}})$, where $d$ is the dimension of the hidden state.
As shown in the leftmost column of Fig.~\ref{fig:method}, we first sample $N$ unlabeled queries $\{x_j\}_{j=1}^N$ from the training set, and use them to compute the hidden states $h$.
Let $\mH \in \mathbb{R}^{d \times N}$ be the matrix whose $j$-th column is $\vh_{\text{zs}}(x_j)$, and let $\mY \in \mathbb{R}^{d \times N}$ be the matrix whose $j$-th column is $(\vh_{\text{icl}}(x_j) - \vh_{\text{zs}}(x_j))$.
To prevent ill-conditioning of $\mH\mH^\top$ when $d$ is large relative to $N$, we employ the ridge regression
\begin{equation}
\label{eq:ridge}
    \mW^{\star} = \underset{\mW}{\operatorname{argmin}}
    \left(
    \|\mY-\mW\mH\|_F^2 +\lambda\|\mW\|_F^2
    \right),
\end{equation}
where $\lambda$ is the regularization parameter, and $\|\cdot\|_F$ denotes the Frobenius norm.
This yields a closed-form solution, computed by solving a linear system
\begin{equation}
\label{eq:closed_form}
    (\mH \mH^\top + \lambda \mI) \mW^{\star\top} = \mH \mY^\top,
\end{equation}
where $\mI$ denotes the identity matrix.

\paragraph{Inference Stage.}
As shown in the rightmost column of Fig.~\ref{fig:method}, we compute the task vector for a test query $x_{\text{test}}$ as $\vv = \mW^{\star} \vh_{\text{zs}}(x_{\text{test}})$.
The task-conditioned hidden state is then obtained as $\vh_{\text{tv}} = \vh_{\text{zs}} + \vv$, from which the LM head computes the predictive distribution $P_{\text{tv}}$.

\section{Experiments}
\label{sec:exp}
In this section, we empirically validate the proposed \oursmethod{} method.
Sec.~\ref{subsec:exp_setup} describes the evaluation setup, and
Sec.~\ref{subsec:exp_results} presents the experimental results.
Codes are available at \href{https://github.com/Jii111/LTV}{this GitHub repository}.

\subsection{Experimental Setup}
\label{subsec:exp_setup} 

\paragraph{Evaluation.}
Following prior work~\citep{li2024implicit, gao2025optimization}, we evaluate \oursmethod{} on a diverse set of text classification benchmarks.
We consider eight widely used datasets: SST-2, SST-5~\citep{socher2013recursive}, MR~\citep{pang2005seeing}, AGNews, DBPedia~\citep{zhang2015character}, TREC~\citep{voorhees2000building}, SUBJ~\citep{pang2004sentimental}, HateSpeech18~\citep{de2018hate}.
\revise{Unless otherwise specified,} we report the average accuracy over five independent runs, each using a different set of demonstrations with $k=30$.
For \oursmethod{}, we report results using $N=256$ train queries and a regularization parameter $\lambda = 5$ selected via grid search.
Further details on evaluation are provided in Appendix~\ref{app:exp_details}.

\paragraph{Models.}
We conduct experiments on five widely used LLMs: LLaMA-2-7B~\citep{touvron2023llama}, LLaMA-2-13B~\citep{touvron2023llama}, LLaMA-3.1-8B~\citep{dubey2024llama}, Qwen-2.5-7B~\citep{qwen2024qwen2}, and Qwen-3-8B~\citep{yang2025qwen3}.

\paragraph{Baselines.}
We compare \oursmethod{} with existing task vector methods in a training-free setting, which can be categorized into three groups based on the module from which task vectors are extracted.
First, \emph{Task Vector}~\citep{hendel2023context} extracts the task vector from the outputs of a single \emph{decoder layer}.
Second, \emph{Function Vector}~\citep{todd2023function} and \emph{State Vector}~\citep{li2024context} extract task vectors from a subset of \emph{attention heads}.
Third, \emph{I2CL}~\citep{li2024implicit} extracts task vectors from both attention heads and MLP modules.
In I2CL, the task vector is injected after scaling, where the scaling coefficient is \emph{trained} using a validation set with true labels.
For fair comparison in a training-free setting, we use the default coefficient selected in~\citep{li2024implicit}, instead of training the coefficient.
We also compare with standard zero-shot and ICL baselines.
Further details on each method are provided in Appendix~\ref{app:baseline_methods}.

\input{table/comparison_with_baseline}

\input{figure/kl_bar}
\input{table/time_efficiency}

\input{table/hyperparameters}
\input{table/regression}
\input{table/transfer}
\input{figure/peft_comparison}

\subsection{Results}
\label{subsec:exp_results}

\paragraph{\oursmethod{} outperforms baselines.}
Table~\ref{tab:comparison_with_baseline} reports classification accuracy of \oursmethod{} and baseline task vector methods across eight benchmarks.
We present results on two representative models here, with full results on all five LLMs provided in Appendix~\ref{app:additional_exp}.
As shown in the rightmost column, \oursmethod{} achieves the highest average accuracy on all five models.
Notably, \oursmethod{} outperforms the best baseline -- \emph{State Vector} -- on LLaMA-2-13B by 9.2\%, achieving a score of 79.46\% compared to 70.22\%.
Moreover, \oursmethod{} ranks first on the majority of individual tasks across all models; specifically, our method achieves the best performance on 31 out of 40 cases measured on 8 benchmarks and 5 models.
This strong performance highlights that, among TV methods, \oursmethod{} most effectively extracts task vectors that capture the effect of demonstrations.
A comparison with I2CL using its original validation and tuning protocol across all five models is provided in Appendix~\ref{app:optimized_i2cl}.

\paragraph{\oursmethod{} effectively reduces the proposed metric.}
Recall that \oursmethod{} is designed to achieve a small \oursmetric{}.
Given its strong performance over baselines, a natural follow-up is whether \oursmethod{} indeed attains lower \oursmetric{}.
Fig.~\ref{fig:kl_bar} compares \oursmetric{} across five methods on eight benchmarks experimented with LLaMA-3.1-8B model.
Compared to other baselines, \oursmethod{} achieves the lowest \oursmetric{} on all benchmarks.
This confirms that \oursmethod{} works as intended, successfully reducing the discrepancy between TV and ICL modes.

\paragraph{\oursmethod{} incurs the lowest inference overhead.}
We evaluate the computational efficiency of \oursmethod{} in terms of one-time extraction cost and per-sample inference latency (Table ~\ref{tab:inference_time}). Although \oursmethod{} requires a higher extraction time compared to some TV baselines, this overhead is a one-time pre-computation cost that is effectively amortized over the entire test set. More importantly, \oursmethod{} achieves the most competitive inference latency (0.0226s per sample) among all considered vector-based methods, matching the speed of zero-shot inference. This efficiency stems from the fact that \oursmethod{} operates solely on the final layer through a single matrix-vector multiplication, providing a significant accuracy boost (75.46\%) without compromising the model's original inference throughput.

\paragraph{\oursmethod{} is robust to hyperparameter choices.}
One might wonder whether the performance of \oursmethod{} relies on careful hyperparameter tuning.
We examine the effect of the number of queries $N$ and the ridge regularization parameter $\lambda$ on both accuracy and \oursmetric{}. 
Table~\ref{tab:hyperparameters} reports results obtained by varying one hyperparameter at a time.
Halving $N$ from the default value of 256 to 128 decreases accuracy by only 0.5\% with a slight increase in \oursmetric{}.
Similarly, varying $\lambda$ from the default value of 5.0 to 1.0 or 10.0 maintains comparable accuracy and \oursmetric{}.
These results confirm that \oursmethod{} is robust to hyperparameter choices.

\linespread{0.97}\selectfont
\vspace{-3pt}
\paragraph{\oursmethod{} also remains effective on regression tasks.}
Most prior studies on task vectors have focused exclusively on classification, leaving it unclear whether existing TV methods are effective in settings with continuous outputs.
We address this gap by evaluating TV methods on regression tasks, following the experimental setup of prior studies~\citep{yang2025task, garg2022can}.
Specifically, we conduct linear regression and ReLU~\citep{agarap2018deep} regression tasks, and report the mean squared error (MSE) of \oursmethod{} and task vector baselines in Table~\ref{tab:regression}.
\oursmethod{} achieves the lowest MSE among all task vector baselines on both tasks (5.13 on linear regression; 3.45 on ReLU regression), showing that our approach is also effective on regression tasks.
Detailed setups for the regression experiments are provided in Appendix~\ref{app:regression}.

\vspace{-3pt}
\paragraph{\oursmethod{} can transfer task vectors across model scales.}
One might be concerned whether task vector methods remain useful even when the inference model cannot perform ICL effectively.
Such cases arise, for instance, when the context length is too short to fit enough demonstrations or when the inference model has limited ICL capability.
\oursmethod{} can address this concern through cross-model transfer, which is enabled by a simple extension (see Appendix~\ref{app:transfer} for details).
This transfer allows us to extract task vectors from a larger, more capable model and apply them to a smaller one.
Notably, as shown in Table~\ref{tab:ltv_transfer}, transferring \oursmethod{} from the 72B model to the 7B model improves the average classification accuracy of the \revise{smaller} model by 6.43\%, even matching the accuracy of the 72B model (79.96\%).
\revise{The change in accuracy after transfer tends to follow the difference in the two models' accuracy under ICL. For example, on Subj and DBPedia, the larger model outperforms the smaller model by the largest margins under ICL, and transfer produces the largest accuracy improvements. On TREC, in contrast, the larger model is less accurate under ICL than the smaller model (75.44\% vs.\ 87.04\%), and transfer decreases accuracy.}
Overall, these results suggest that \oursmethod{} remains useful even when the inference model itself cannot perform ICL effectively, since task vectors can be extracted from a more capable model and transferred to it.

\vspace{-3pt}
\paragraph{\oursmethod{} achieves higher accuracy than PEFT-based approaches \revise{at lower cost in the few-shot regime.}}
We further examine alternative ways of minimizing our proposed surrogate $\mathcal{L}_{\text{MSE}}$ by employing parameter-efficient fine-tuning (PEFT) methods, instead of the closed-form linear mapping used in \oursmethod{}.
Specifically, we consider Prompt Tuning~\citep{lester2021power} and LoRA~\citep{hu2022lora}, two widely used PEFT approaches, and train them to minimize the same $\mathcal{L}_{\text{MSE}}$ objective in~\eqref{eq:mse}.
We evaluate all methods under three demonstration budgets ($k=10, 30, 50$), and report the extraction (training for PEFT) time and the classification accuracy averaged across 8 benchmarks in Fig.~\ref{fig:peft_comparison}.
In the few-shot regime ($k=10$), \oursmethod{} outperforms LoRA in both accuracy and \revise{time efficiency}.
When more demonstrations are available ($k=30, 50$), LoRA achieves higher accuracy than \oursmethod{}, \revise{but its training time is roughly $3$ times \oursmethod{}'s extraction time}.
These results indicate that our method provides a competitive edge compared to PEFT-based methods, particularly in the low-data regime.

\vspace{-3pt}
\section{Conclusion}
\vspace{-3pt}

In this paper, we investigated the mechanics of task vectors (TVs) as a means of compressing In-Context Learning (ICL) demonstrations into the hidden states of LLMs. We introduced $d_{\text{NTP}}$, a metric that quantifies the discrepancy in next-token probability between TV-based and standard ICL-based inference. Our analysis reveals that $d_{\text{NTP}}$ serves as a proxy for downstream performance, enabling the evaluation of TV methods without the need for exhaustive task-specific testing. Leveraging these insights, we developed Linear Task Vector (LTV), a method that employs a linear mapping to minimize $d_{\text{NTP}}$. Experimental results across multiple benchmarks and architectures demonstrate that LTV outperforms existing training-free TV methods. By achieving superior accuracy with reduced inference latency, LTV offers a more efficient and effective framework for task compression in LLMs.

\linespread{1}\selectfont
\section*{Acknowledgements}
This work was supported by the National Research Foundation of Korea (NRF) grant funded by the Korea government (MSIT) (No. RS-2024-00345351, No. RS-2024-00408003); by the Institute of Information \& Communications Technology Planning \& Evaluation (IITP) grant funded by the Korea government (MSIT) (No. RS-2025-02283048, Developing the Next-Generation General AI with Reliability, Ethics, and Adaptability); by the IITP under the Leading Generative AI Human Resources Development (IITP-2026-RS-2026-25544647) grant funded by the Korea government (MSIT); and by the National Institute for Mathematical Sciences (NIMS) grant funded by the Korea government (MSIT) (No. S26120000).

\bibliography{main}
\bibliographystyle{plainnat}

%% file: figure/modes.tex
\begin{figure}[t]
    \centering
\includegraphics[width=1.0\linewidth]{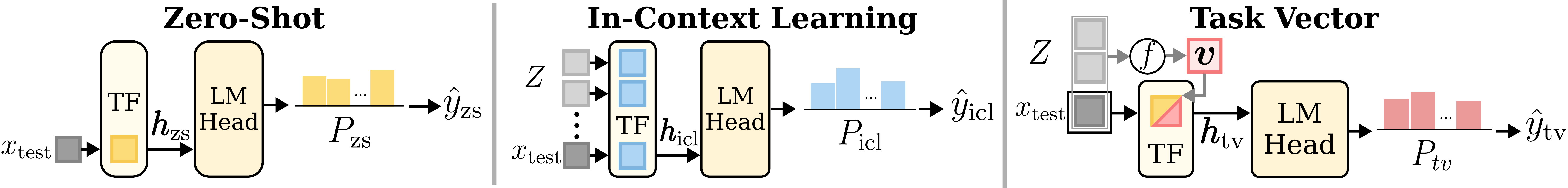}
    \caption{
Comparison of three inference modes.
In \emph{zero-shot} inference mode (\textbf{left}), the model predicts the next token $\hat{y}_{\text{zs}}$ 
solely based on the test query $x_{\text{test}}$.
In the \emph{In-Context Learning} mode (\textbf{middle}), the model predicts the next token $\hat{y}_{\text{icl}}$  based on the concatenation of demonstrations $Z$ and the query $x_{\text{test}}$.
In the \emph{task vector} mode (\textbf{right}), the model predicts the next token $\hat{y}_{\text{tv}}$ based on 
not only the query $x_{\text{test}}$, but also an injected task vector $\vv=f(Z)$, which is added to the model activation.
Here, the task vector $\vv$ is constructed by a function $f$ using the demonstrations $Z$.
}
\label{fig:modes}
\end{figure}

%% file: figure/metric.tex
\begin{figure}[t]
    \centering
\includegraphics[width=0.8\linewidth]{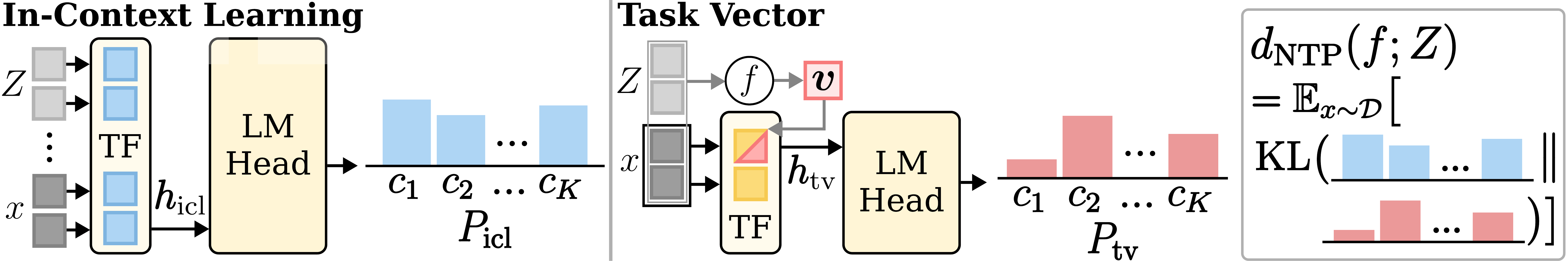}
    \caption{
Overview of the proposed metric $d_{\text{NTP}}(f; Z)$ in~\eqref{eq:metric}, which measures the quality of the task vector extraction method $f$.
In the ICL mode, the model gets demonstrations $Z$ and test query $x_{\text{test}}$ together to estimate the probability distribution $P_{\text{icl}}$ for the next token (\textbf{left}).
In the TV mode, the task vector $\vv$ is injected in the hidden layer (instead of putting  demonstrations $Z$ in the input layer) to get the distribution $P_{\text{tv}}$ for the next token (\textbf{right}). Our proposed metric measures the expected Kullback-Leibler (KL) divergence between these two distributions ($P_{\text{icl}}$ and $P_{\text{tv}}$), thus checking whether the effect of using $\vv=f(Z)$ is equivalent to the effect of using $Z$ in the input.
}
\label{fig:metric}
\end{figure}

%% file: figure/fig_corr.tex
\begin{figure*}[t]
    \centering
\includegraphics[width=1.0\textwidth]{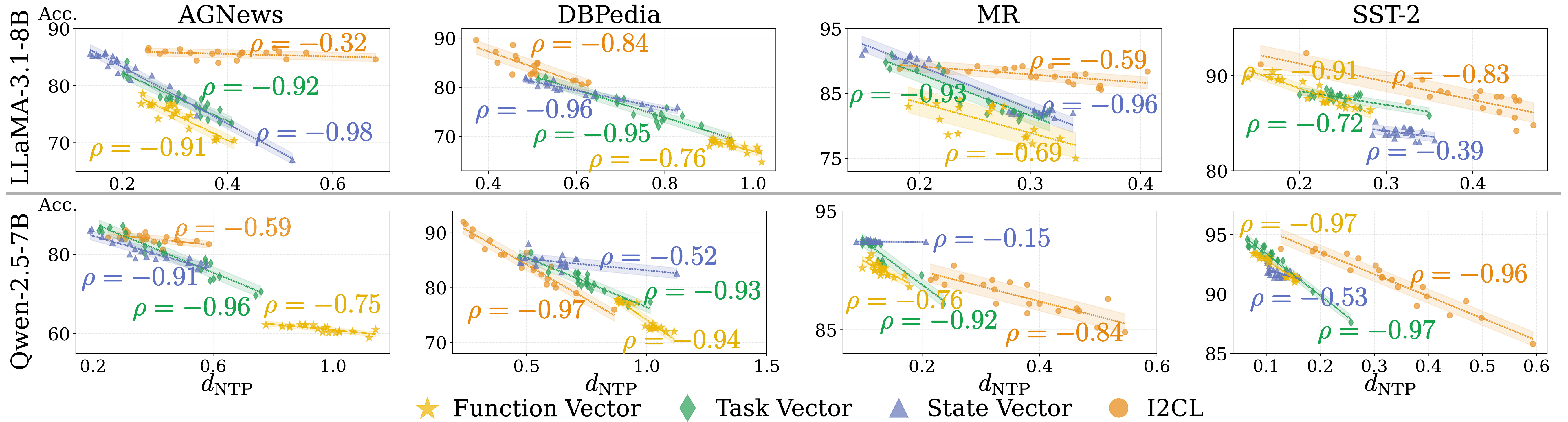}
    \caption{
    Correlation between the proposed discrepancy metric $d_{\text{NTP}}(f; Z)$ and the test accuracy, measured on various task vector extraction methods $f$ and the demonstrations $Z$.
    We test on four variants of $f$: Function Vector~\citep{todd2023function}, Task Vector~\citep{hendel2023context}, State Vector~\citep{li2024context}, and I2CL~\citep{li2024implicit}, each of which is shown in different colors. Here, each point corresponds to the result for different demonstrations $Z$.
    Pearson correlation coefficients $(\rho)$ are reported for each method $f$.
    Across four classification benchmarks (columns) and two models (rows), lower $d_{\text{NTP}}(f; Z)$ consistently correlates with higher accuracy, validating our metric as a principled criterion for the quality of the task vector.
    }
\label{fig:fig_corr}
\end{figure*}

%% file: figure/method.tex
\begin{figure*}[t]
    \centering
\includegraphics[width=1.0\textwidth]{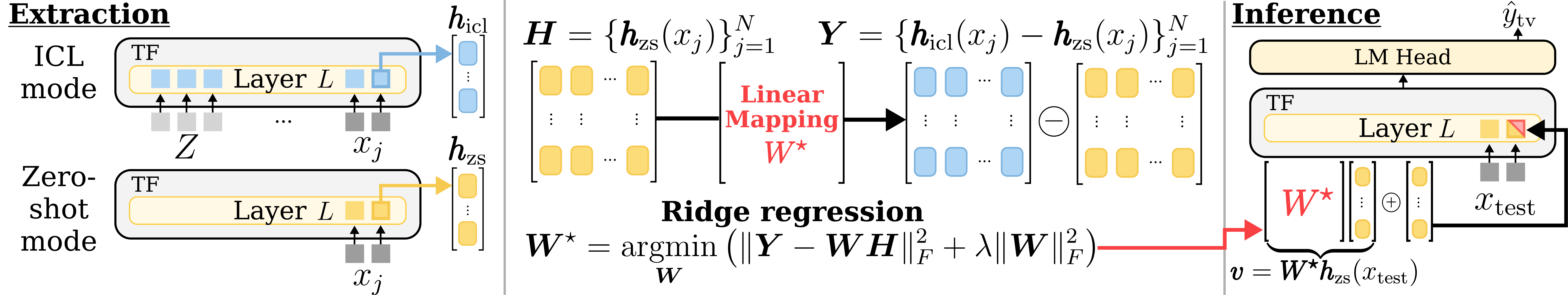}
\caption{
Overview of our Linear Task Vector (\oursmethod{}) method.
Our method employs a linear mapping $\mW$ that estimates the effect of demonstrations in the hidden space $(\vh_{\text{icl}} - \vh_{\text{zs}})$ from the hidden state $\vh_{\text{zs}}$ of the zero-shot inference mode via ridge regression.
In the extraction phase \textbf{(left)}, we use $N$ unlabeled training queries $\{x_j\}_{j=1}^N$ to define (1) the regression target matrix $\mY$ as the concatenation of $N$ column vectors $\{\vh_{\text{icl}}(x_j) - \vh_{\text{zs}}(x_j)\}_{j=1}^N$ and (2) the variable matrix $\mH$ as the concatenation of $N$ column vectors $\{\vh_{\text{zs}}(x_j)\}_{j=1}^N$.
Subsequently, a closed-form solution for the optimal linear mapping $\mW^{\star}$ from $\mH$ to $\mY$
is obtained.
In the inference phase \textbf{(right)}, we use the mapping $\mW^{\star}$ to inject the task vector $\vv=\mW^{\star} \vh_{\text{zs}}(x_{\text{test}})$ in the hidden state, 
for a given test query $x_{\text{test}}$.
}
\label{fig:method}
\end{figure*}

%% file: table/comparison_with_baseline.tex
\newcommand{\bg}[1]{\cellcolor{gray!15}#1}

\begin{table*}[t]
    \LARGE
    \centering
    \caption{
Classification accuracy (\%) of \oursmethod{} and four TV methods, tested on eight benchmarks. We also compare with two inference modes (zero-shot and ICL) for reference.
\oursmethod{} method achieves the highest average accuracy (Avg.) across all models; check full results on five LLMs in Appendix~\ref{app:additional_exp}.
}
    \resizebox{1.0\textwidth}{!}{
    {\setlength{\tabcolsep}{10pt}
    \begin{tabular}{l l| c c c c c c c c| c}
        \toprule
        Models & Methods
        & AGNews & DBPedia & HateSpeech18 & MR & SST-2 & SST-5 & Subj & TREC & Avg. \\
        \midrule

        \multirow{7}{*}{\shortstack{\textit{LLaMA-2-13B}\\{\large\citep{touvron2023llama}}}}
        & Zero-shot
          & 73.80 & 76.20 & 55.20 & 61.20 & 66.60 & 20.80 & 49.80 & 69.60 & 59.15 \\
        & ICL
          & 88.04 {\large ($\pm$1.7)}
          & 96.88 {\large ($\pm$1.1)}
          & 77.40 {\large ($\pm$2.2)}
          & 94.44 {\large ($\pm$0.5)}
          & 94.80 {\large ($\pm$0.7)}
          & 46.52 {\large ($\pm$2.7)}
          & 82.28 {\large ($\pm$7.1)}
          & 83.36 {\large ($\pm$5.2)}
          & 82.97 \\
          \cmidrule(lr){2-11}
        & Function Vector{\large~\citep{todd2023function}}
          & 74.00 {\large ($\pm$1.2)}
          & 76.40 {\large ($\pm$7.5)}
          & 54.08 {\large ($\pm$0.4)}
          & 61.84 {\large ($\pm$3.4)}
          & 72.44 {\large ($\pm$6.0)}
          & 21.40 {\large ($\pm$1.2)}
          & 50.00 {\large ($\pm$0.0)}
          & 69.96 {\large ($\pm$1.1)}
          & 60.02 \\
        & Task Vector{\large~\citep{hendel2023context}}
          & 79.84 {\large ($\pm$2.9)}
          & 80.72 {\large ($\pm$1.9)}
          & \underline{71.48} {\large ($\pm$6.7)}
          & 83.16 {\large ($\pm$0.8)}
          & 82.80 {\large ($\pm$4.2)}
          & 33.64 {\large ($\pm$2.1)}
          & \underline{49.72} {\large ($\pm$0.1)}
          & \underline{73.84} {\large ($\pm$5.3)}
          & 69.40 \\
        & State Vector{\large~\citep{li2024context}}
          & \underline{84.64} {\large ($\pm$4.0)}
          & \underline{89.48} {\large ($\pm$4.9)}
          & 58.80 {\large ($\pm$8.8)}
          & \underline{89.20} {\large ($\pm$2.8)}
          & \underline{87.20} {\large ($\pm$3.6)}
          & \underline{36.08} {\large ($\pm$3.0)}
          & 49.40 {\large ($\pm$0.6)}
          & 66.96 {\large ($\pm$1.0)}
          & \underline{70.22} \\
        & I2CL{\large~\citep{li2024implicit}}
          & 78.88 {\large ($\pm$0.3)}
          & 79.00 {\large ($\pm$0.5)}
          & 54.48 {\large ($\pm$0.1)}
          & 60.68 {\large ($\pm$0.3)}
          & 64.80 {\large ($\pm$0.4)}
          & 25.96 {\large ($\pm$0.1)}
          & 50.00 {\large ($\pm$0.0)}
          & 69.12 {\large ($\pm$0.4)}
          & 60.37 \\
        & \bg{\textbf{\oursmethod{} (Ours)}}
          & \bg{\textbf{86.68} {\large ($\pm$1.7)}}
          & \bg{\textbf{93.20} {\large ($\pm$0.5)}}
          & \bg{\textbf{72.08} {\large ($\pm$2.2)}}
          & \bg{\textbf{90.20} {\large ($\pm$2.0)}}
          & \bg{\textbf{88.96} {\large ($\pm$2.1)}}
          & \bg{\textbf{41.88} {\large ($\pm$2.2)}}
          & \bg{\textbf{78.76} {\large ($\pm$2.7)}}
          & \bg{\textbf{83.92} {\large ($\pm$6.3)}}
          & \bg{\textbf{79.46}} \\
        \midrule

        \multirow{7}{*}{\shortstack{\textit{LLaMA-3.1-8B}\\{\large\citep{dubey2024llama}}}}
        & Zero-shot
          & 75.00 & 69.00 & 60.60 & 82.20 & 86.80 & 25.60 & 59.20 & 49.40 & 63.48 \\
        & ICL
          & 87.16 {\large ($\pm$1.2)}
          & 97.68 {\large ($\pm$0.8)}
          & 74.32 {\large ($\pm$8.1)}
          & 94.52 {\large ($\pm$0.4)}
          & 94.20 {\large ($\pm$1.1)}
          & 48.32 {\large ($\pm$1.1)}
          & 85.96 {\large ($\pm$5.3)}
          & 79.08 {\large ($\pm$7.3)}
          & 82.66 \\
          \cmidrule(lr){2-11}
        & Function Vector{\large~\citep{todd2023function}}
          & 76.16 {\large ($\pm$0.1)}
          & 69.44 {\large ($\pm$1.0)}
          & 63.00 {\large ($\pm$0.2)}
          & 83.24 {\large ($\pm$0.3)}
          & 87.32 {\large ($\pm$0.6)}
          & 26.20 {\large ($\pm$1.3)}
          & 59.52 {\large ($\pm$0.2)}
          & \underline{69.12} {\large ($\pm$0.6)}
          & 66.75 \\
        & Task Vector{\large~\citep{hendel2023context}}
          & \underline{81.36} {\large ($\pm$3.7)}
          & \underline{83.24} {\large ($\pm$1.5)}
          & \underline{67.64} {\large ($\pm$2.0)}
          & 83.48 {\large ($\pm$3.5)}
          & 87.88 {\large ($\pm$0.3)}
          & 34.76 {\large ($\pm$1.9)}
          & 61.12 {\large ($\pm$1.4)}
          & 66.48 {\large ($\pm$1.0)}
          & \underline{70.75} \\
        & State Vector{\large~\citep{li2024context}}
          & 80.28 {\large ($\pm$4.3)}
          & 80.80 {\large ($\pm$1.5)}
          & 65.40 {\large ($\pm$1.3)}
          & \underline{85.96} {\large ($\pm$4.9)}
          & 84.12 {\large ($\pm$0.6)}
          & \underline{36.60} {\large ($\pm$1.6)}
          & \underline{62.52} {\large ($\pm$3.5)}
          & 67.28 {\large ($\pm$1.9)}
          & 70.37 \\
        & I2CL{\large~\citep{li2024implicit}}
          & 76.76 {\large ($\pm$0.4)}
          & 72.56 {\large ($\pm$0.4)}
          & 62.24 {\large ($\pm$0.5)}
          & 85.24 {\large ($\pm$0.3)}
          & \textbf{90.80} {\large ($\pm$0.2)}
          & 32.48 {\large ($\pm$0.4)}
          & 62.28 {\large ($\pm$0.1)}
          & 49.00 {\large ($\pm$0.5)}
          & 66.42 \\
        & \bg{\textbf{\oursmethod{} (Ours)}}
          & \bg{\textbf{82.84} {\large ($\pm$2.5)}}
          & \bg{\textbf{93.36} {\large ($\pm$1.4)}}
          & \bg{\textbf{70.44} {\large ($\pm$6.0)}}
          & \bg{\textbf{88.68} {\large ($\pm$1.0)}}
          & \bg{\underline{90.08} {\large ($\pm$1.2)}}
          & \bg{\textbf{38.20} {\large ($\pm$3.4)}}
          & \bg{\textbf{67.16} {\large ($\pm$11.4)}}
          & \bg{\textbf{72.88} {\large ($\pm$7.2)}}
          & \bg{\textbf{75.46}} \\
        \bottomrule
    \end{tabular}
    }
    }
    \label{tab:comparison_with_baseline}
\end{table*}

%% file: figure/kl_bar.tex
\begin{figure*}[t]
    \vspace{3pt}
    \centering
\includegraphics[width=0.9\textwidth]{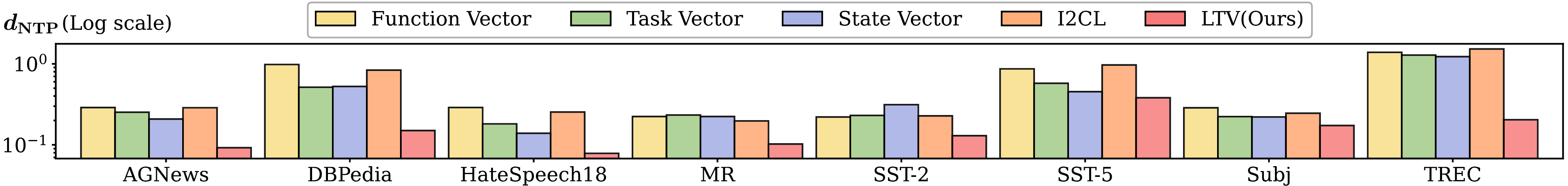}
    \caption{
    Comparison of \oursmetric{} across \oursmethod{} and four baselines on eight benchmarks, tested on LLaMA-3.1-8B.
    \oursmethod{} consistently achieves the lowest \oursmetric{} across all benchmarks.
    }
\label{fig:kl_bar}
\vspace{-6pt}
\end{figure*}

%% file: table/time_efficiency.tex
\begin{table}[t]
    \centering
    \caption{
    Extraction and inference latency of \oursmethod{} and task vector baselines, evaluated on LLaMA-3.1-8B.
    Classification accuracy (\%) is reported for reference.
    }
    \setlength{\tabcolsep}{8pt}
    \resizebox{\linewidth}{!}
    {\scriptsize
    \begin{tabular}{l|cc|ccccc}
        \toprule
        & Zero-shot & ICL 
        & Function Vector 
        & I2CL 
        & State Vector
        & Task Vector 
        & \cellcolor{gray!15}\textbf{\oursmethod{} (Ours)} \\
        \midrule
        Extraction (s) $\downarrow$ 
            & - & - & 198.965 & \textbf{1.347} & \underline{27.756} & 27.773 
            & \cellcolor{gray!15}42.664 \\
        Inference (s) $\downarrow$ 
            & 0.0220 & 0.2238 & 0.0614 & 0.0246 & 0.0292 & \underline{0.0236} 
            & \cellcolor{gray!15}\textbf{0.0226} \\
        Accuracy (\%) $\uparrow$ 
            & 63.48 & 82.66 & 66.75 & 66.42 & 70.37 & \underline{70.75} 
            & \cellcolor{gray!15}\textbf{75.46} \\
        \bottomrule
    \end{tabular}
    }
    \label{tab:inference_time}
    \vspace{2pt}
\end{table}

%% file: table/hyperparameters.tex
\begin{table}[t]
\centering
\caption{
Effect of hyperparameters (the number of unlabeled queries $N$ and the regularization parameter $\lambda$) on the performance of \oursmethod{} and \oursmetric{}.
Experiments are conducted on LLaMA-3.1-8B, with gray columns indicating default settings.
}
\setlength{\tabcolsep}{18pt}
\resizebox{\linewidth}{!}
{\small
\begin{tabular}{l|ccc|ccc}
\toprule
& \multicolumn{3}{c|}{$N$} & \multicolumn{3}{c}{$\lambda$} \\
\cmidrule(lr){2-4} \cmidrule(lr){5-7}
& 64 & 128 & \cellcolor{gray!15}256 
& 1.0 & \cellcolor{gray!15}5.0 & 10.0 \\
\midrule
Avg. Acc. (\%) $\uparrow$ 
& 72.2 & 74.7 & \cellcolor{gray!15}\textbf{75.2} 
& 75.2 & \cellcolor{gray!15}\textbf{75.2} & 75.2 \\
\oursmetric{} $\downarrow$ 
& 0.229 & 0.178 & \cellcolor{gray!15}\textbf{0.147} 
& 0.149 & \cellcolor{gray!15}\textbf{0.147} & 0.145 \\
\bottomrule
\end{tabular}
}
\label{tab:hyperparameters}
\end{table}

%% file: table/regression.tex
\begin{table}[t]
    \vspace{-4pt}
    \centering
    \caption{
    Regression performance of \oursmethod{} and task vector baselines, measured by mean squared error (MSE), on linear regression and ReLU regression tasks.
    Experiments are conducted on LLaMA-3.1-8B.
    Zero-shot and ICL are shown for reference.
    }
    \setlength{\tabcolsep}{4pt}
    \resizebox{\linewidth}{!}
    {\scriptsize
    \begin{tabular}{l | cc | ccccc}
        \toprule
        & Zero-shot & ICL 
        & Function Vector 
        & I2CL 
        & State Vector 
        & Task Vector 
        & \cellcolor{gray!15}\textbf{\oursmethod{} (Ours)} \\
        \midrule
        Linear Regression $\downarrow$ 
            & 5.51 & 3.97 & \underline{5.23} & 5.29 & 5.35 & 5.97 
            & \cellcolor{gray!15}\textbf{5.13} \\
        ReLU Regression $\downarrow$ 
            & 3.82 & 3.33 & \underline{3.63} & 3.68 & 3.80 & 4.05 
            & \cellcolor{gray!15}\textbf{3.45} \\
        \bottomrule
    \end{tabular}
    }
    \label{tab:regression}
    \vspace{2pt}
\end{table}

%% file: table/transfer.tex
\begin{table*}[t]
    \LARGE
    \centering
    \caption{
    Classification accuracy (\%) of \oursmethod{} when task vectors obtained from a \revise{larger} model are transferred and applied to a \revise{smaller} model for task vector-based inference. We use Qwen-2.5-72B as the \revise{larger} model and Qwen-2.5-7B as the \revise{smaller} model. As a reference, we report zero-shot, ICL, and \oursmethod{} inference of each model using its own task vectors. Values in parentheses indicate the change relative to the \oursmethod{} performance of the \revise{smaller} model.
    }
    \label{tab:ltv_transfer}
    \resizebox{1.0\textwidth}{!}{
    {\setlength{\tabcolsep}{5pt}
    \begin{tabular}{l l| c c c c c c c c| c}
        \toprule
        Models & Methods
        & AGNews & DBPedia & HateSpeech18 & MR & SST-2 & SST-5 & Subj & TREC & Avg. \\
        \midrule
        \multirow{3}{*}{\shortstack{{\Large\textit{Smaller Model}}\\ \large{(\textit{Qwen-2.5-7B}\,\citep{qwen2024qwen2})}}}
        & Zero-shot
          & 74.20 & 73.20 & 61.20 & 62.20 & 55.40 & 20.00 & 56.80 & 67.20 & 58.78 \\
        & ICL
          & 83.24 {\large ($\pm$2.7)} & 83.48 {\large ($\pm$4.1)} & 80.60 {\large ($\pm$1.3)} & 92.88 {\large ($\pm$0.2)} & 95.00 {\large ($\pm$0.9)} & 46.60 {\large ($\pm$2.9)} & 76.00 {\large ($\pm$6.8)} & 87.04 {\large ($\pm$1.5)} & 80.61 \\
        & \oursmethod{}
          & 78.48 {\large ($\pm$4.8)} & 77.24 {\large ($\pm$1.9)} & 72.28 {\large ($\pm$2.7)} & 87.56 {\large ($\pm$4.1)} & 89.00 {\large ($\pm$7.4)} & 33.16 {\large ($\pm$2.6)} & 66.52 {\large ($\pm$8.4)} & 84.32 {\large ($\pm$4.0)} & 73.57 \\
        \midrule
        \multirow{3}{*}{\shortstack{{\Large\textit{Larger Model}}\\ \large{(\textit{Qwen-2.5-72B}\,\citep{qwen2024qwen2})}}}
        & Zero-shot
          & 72.00 & 68.00 & 70.40 & 92.00 & 90.60 & 36.00 & 53.00 & 73.00 & 69.38 \\
        & ICL
          & 88.08 {\large ($\pm$1.1)} & 95.64 {\large ($\pm$0.6)} & 82.80 {\large ($\pm$1.4)} & 93.88 {\large ($\pm$0.3)} & 96.44 {\large ($\pm$0.2)} & 46.12 {\large ($\pm$1.6)} & 95.80 {\large ($\pm$0.7)} & 75.44 {\large ($\pm$7.1)} & 84.28 \\
        & \oursmethod{}
          & 81.40 {\large ($\pm$1.4)} & 88.92 {\large ($\pm$1.9)} & 77.20 {\large ($\pm$1.6)} & 91.48 {\large ($\pm$0.4)} & 94.36 {\large ($\pm$1.0)} & 37.16 {\large ($\pm$2.8)} & 91.92 {\large ($\pm$1.3)} & 77.24 {\large ($\pm$3.8)} & 79.96 \\
        \midrule
        \rowcolor{gray!15}
        &
        & \textbf{83.24} {\large ($\pm$1.2)}
        & \textbf{92.68} {\large ($\pm$1.8)}
        & \textbf{76.48} {\large ($\pm$1.9)}
        & \textbf{91.36} {\large ($\pm$0.5)}
        & \textbf{92.36} {\large ($\pm$1.1)}
        & \textbf{42.48} {\large ($\pm$2.7)}
        & \textbf{90.24} {\large ($\pm$1.5)}
        & 71.16 {\large ($\pm$5.9)}
        & \textbf{80.00} \\
        \rowcolor{gray!15}
        \multirow{-2}{*}{\shortstack{{\Large\textit{Smaller Model}}\\ \large{(\textit{Qwen-2.5-7B}\,\citep{qwen2024qwen2})}}}
        & \multirow{-2}{*}{\shortstack[l]{Transferred\\ \oursmethod{} \large{(72B $\rightarrow$ 7B)}}}
        & \rule{0pt}{5pt}\large\textcolor{red}{(+4.76)}
        & \large\textcolor{red}{(+15.44)}
        & \large\textcolor{red}{(+4.20)}
        & \large\textcolor{red}{(+3.80)}
        & \large\textcolor{red}{(+3.36)}
        & \large\textcolor{red}{(+9.32)}
        & \large\textcolor{red}{(+23.72)}
        & \large\textcolor{blue1}{(-13.16)}
        & \large\textcolor{red}{(+6.43)} \\
        \bottomrule
    \end{tabular}
    }
}
\vspace{2pt}
\end{table*}

%% file: figure/peft_comparison.tex
\begin{figure*}[t]
    \centering
\includegraphics[width=0.95\textwidth]{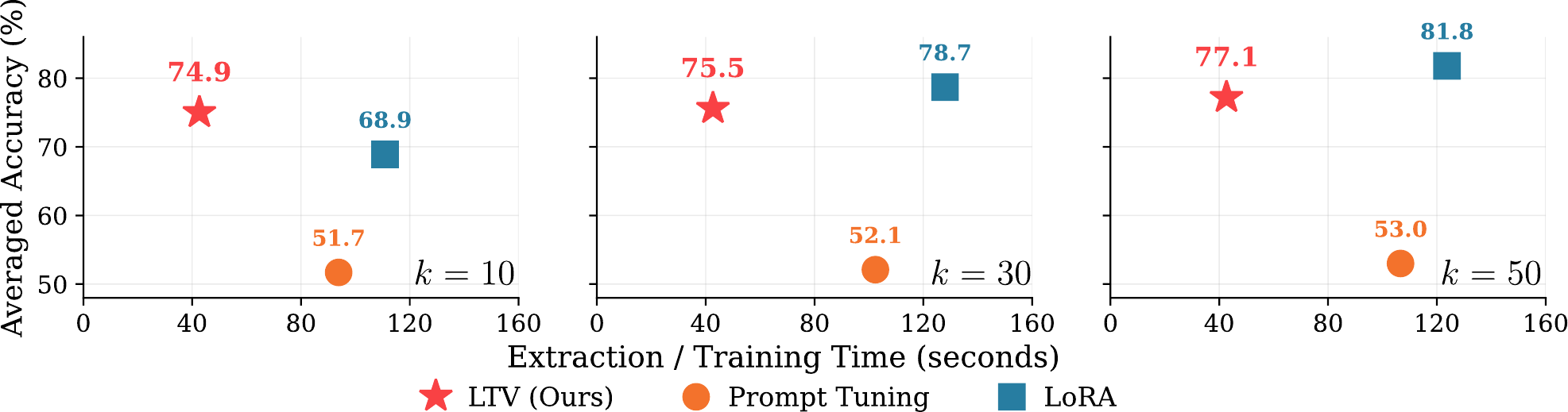}
    
    \caption{
    \revise{Extraction (training for PEFT) time} versus downstream accuracy of \oursmethod{} and PEFT-based methods (prompt tuning~\citep{lester2021power} and LoRA~\citep{hu2022lora}), evaluated under three numbers of demonstrations ($k=10, 30, 50$).
    All methods are tested on LLaMA-3.1-8B.
}
\label{fig:peft_comparison}
\end{figure*}

%% file: appendix.tex
\appendix

\section{Appendix}

\noindent\begin{minipage}{\textwidth}
\captionsetup{type=table}
{\input{table/benchmark_details}}
\end{minipage}

\subsection{Implementation Details on Evaluation}
\label{app:exp_details}

We provide detailed information about our evaluation setup.
Following experiments conducted in~\citet{li2024implicit}, we employ the same prompt templates and label sets, as summarized in Table~\ref{tab:benchmark_details}.
However, EmoC~\citep{chatterjee2019semeval} is excluded as it is no longer publicly available.
For labels consisting of multiple tokens, we use the probability of the first token when computing~\eqref{eq:restricted_prob}, following~\citet{hendel2023context, todd2023function, li2024implicit}.
All experiments were conducted using a single NVIDIA GPU (either RTX 6000 Ada Generation or GeForce RTX 4090).

All experiments are conducted with $k=30$ demonstrations $Z$, sampled from the training dataset.
Note that for classification tasks, the label set $\mathcal{C}$ contains $\lvert \mathcal{C} \rvert = K$ labels.
For both task vector extraction and ICL inference, we ensure that each label is represented by an equal number of examples in $Z$.
Specifically, we include the maximum number of examples per label such that all labels are equally represented while staying within the limit of $k$ demonstrations.
For example, in the case of AGNews where the number of classes $K$ is 4, we include 7 examples per label under the $k=30$ setting, resulting in a total of 28 demonstrations.
This uniform label distribution is maintained across all experiments involving demonstrations.

\newpage
\subsection{Proof of Proposition~\ref{prop:kl_bound}}
\label{app:proof_kl_bound}

We provide detailed proof for the proposition~\ref{prop:kl_bound} presented in Sec.~\ref{sec:method}.

\begin{proof}
Recall that, throughout the paper, $P$ denotes the next-token distribution
\emph{restricted} to the label set $\mathcal{C}$ as defined in~\eqref{eq:restricted_prob}.
Accordingly, the $P_{\text{icl}}$ and $P_{\text{tv}}$ in the KL divergence in \oursmetric{}$(f)$ in~\eqref{eq:metric} are also computed over $\mathcal{C}$.

Given demonstrations $Z$ and a query $x$, recall that
\[
\vh_{\text{icl}} \coloneqq \vh_{\text{icl}}(x,Z), \qquad
\vh_{\text{tv}} \coloneqq \vh_{\text{tv}}(x,\vv)
\]
denote the hidden states of the last token in the final-layer under ICL and TV inference, respectively.

\textbf{Restricted logits.}
For the LM head $\mW_{\text{lm}}$, define $\mW_{\mathcal{C}} \in \mathbb{R}^{K\times d}$ as the submatrix of $\mW_{\text{lm}}$
formed by selecting the rows indexed by the label set $\mathcal{C}$
(where $K = |\mathcal{C}|$).
Then
\[
\|\mW_{\mathcal{C}}\|_2 \le \|\mW_{\text{lm}}\|_2 \le C_1,
\]
since removing rows cannot increase the spectral norm.

Define the \emph{restricted} logits
\[
\vz \coloneqq \mW_{\mathcal{C}}\, \vh_{\text{icl}}, \qquad
\tilde{\vz} \coloneqq \mW_{\mathcal{C}}\, \vh_{\text{tv}},
\]
and the corresponding restricted predictive distributions
\[
\vp \coloneqq \softmax(\vz), \qquad
\vq \coloneqq \softmax(\tilde{\vz}).
\]
Then the discrepancy for each query $x$ appearing in $\text{\oursmetric}(f)$ in~\eqref{eq:metric} is equal to $\KL(\vp\|\vq)$.

\textbf{Step 1: Bounding the restricted logit distance.}
By sub-multiplicativity and the spectral-norm bound,
\begin{equation}
\label{eq:restricted_logit_bound_app}
\|\vz-\tilde{\vz}\|_2
= \|\mW_{\mathcal{C}}(\vh_{\text{icl}}-\vh_{\text{tv}})\|_2
\le \|\mW_{\mathcal{C}}\|_2\,\|\vh_{\text{icl}}-\vh_{\text{tv}}\|_2
\le C_1\,\|\vh_{\text{icl}}-\vh_{\text{tv}}\|_2.
\end{equation}

\textbf{Step 2: Bounding the restricted KL divergence via Lipschitz log-softmax.}
Let $\phi(\cdot) \coloneqq \log\softmax(\cdot)$ denote the log-softmax map on $\mathbb{R}^K$.
By assumption, $\phi$ is $C_2$-Lipschitz in $\ell_2$:
\begin{equation}
\label{eq:lipschitz_logsoftmax_app}
\|\phi(\vz)-\phi(\tilde{\vz})\|_2 \le C_2 \|\vz-\tilde{\vz}\|_2.
\end{equation}
For any coordinate $i\in\{1,\dots,K\}$,
\[
|\log \vp_i - \log \vq_i|
= |\phi(\vz)_i - \phi(\tilde{\vz})_i|
\le \|\phi(\vz)-\phi(\tilde{\vz})\|_2
\le C_2 \|\vz-\tilde{\vz}\|_2.
\]
Therefore,
\begin{align}
\KL(\vp\|\vq)
&= \sum_{i=1}^K \vp_i(\log \vp_i - \log \vq_i)
\le \sum_{i=1}^K \vp_i\,|\log \vp_i - \log \vq_i| \nonumber\\
&\le \sum_{i=1}^K \vp_i \, (C_2 \|\vz-\tilde{\vz}\|_2)
= C_2 \|\vz-\tilde{\vz}\|_2. \label{eq:kl_to_logit_app}
\end{align}

\textbf{Step 3: Combining and taking expectation.}
Combining \eqref{eq:restricted_logit_bound_app} and \eqref{eq:kl_to_logit_app} yields
\[
\KL(\vp\|\vq) \le C_1 C_2 \, \|\vh_{\text{icl}}-\vh_{\text{tv}}\|_2.
\]
Taking expectation over $x\sim\mathcal{D}$, we have
\begin{align*}
\text{\oursmetric}(f)
= \mathbb{E}_{x\sim\mathcal{D}}\big[\KL(\vp\|\vq)\big]
&\le C_1 C_2 \, \mathbb{E}_{x\sim\mathcal{D}}\big[\|\vh_{\text{icl}}-\vh_{\text{tv}}\|_2\big] \\
&\le C_1 C_2 \, \sqrt{\mathbb{E}_{x\sim\mathcal{D}}\big[\|\vh_{\text{icl}}-\vh_{\text{tv}}\|_2^2\big]} \\
&= C_1 C_2 \sqrt{\mathcal{L}_{\text{MSE}}(f)}.
\end{align*}
This completes the proof.
\end{proof}

\newpage
\subsection{Design Choices for the Task Vector}
\label{app:design_choices}
In Sec.~\ref{subsec:ltv_rationale} of the main body, we propose to solve the proxy problem in~\eqref{eq:lin_opt} by minimizing $\mathcal{L}_{\text{MSE}}$ in~\eqref{eq:mse}.
We interpret the proxy problem as finding a task vector $\vv = f(x,Z)$ that compensates for the effect of demonstrations in the hidden state. Specifically, the goal of using task vectors is to approximate the hidden state difference $\vh_{\text{icl}} - \vh_{\text{zs}}$ between ICL and zero-shot inference
\begin{equation*}
    \min_f \mathbb{E}_{x\sim \mathcal{D}} \Big[ \big\| \vh_{\text{icl}}(x, Z) - \vh_{\text{zs}}(x) - f(x,Z) \big\|_2^2 \Big].
\end{equation*}

In the main paper, we consider $f$ as a mapping from $\vh_{\text{zs}}$ to the target $\vh_{\text{icl}} - \vh_{\text{zs}}$ and adopt a linear mapping $f(x,Z) = \mW \vh_{\text{zs}}(x)$ for its closed-form solution
\begin{equation*}
    \min_{\mW} \mathbb{E}_{x\sim \mathcal{D}} \Big[ \big\| \vh_{\text{icl}}(x, Z) - \vh_{\text{zs}}(x) - \mW \vh_{\text{zs}}(x) \big\|_2^2 \Big].
\end{equation*}
Here, we compare this choice against alternative designs and validate its effectiveness.

Given the strong performance of our proposed \oursmethod{} method as shown in Sec.~\ref{subsec:exp_results}, the choice of employing a linear mapping naturally raises two questions: (1) Is it necessary to model $f$ as a mapping, the output of which is dependent on the input $\vh_{\text{zs}}(x)$, or does modeling $f$ as a simple constant vector suffice? (2) How effective is our linear mapping compared to a more expressive alternative trained via gradient descent?
To answer these two questions, we compare \revise{three mapping designs within our proposed framework: a constant mapping, a linear mapping, and} a more expressive alternative (2-layer MLP-based mapping):
\begin{itemize}[leftmargin=*, itemsep=2pt]
    \item \textbf{Constant Mapping:} Modeling the output of $f$ as a fixed vector $\vc \in \mathbb{R}^d$ independent of the query $x$, \ie $f(x,Z) = \vc$.
    The optimal solution is given by the empirical mean of $\vh_{\text{icl}} - \vh_{\text{zs}}$, defined as 
    \begin{equation*}
        \vc^* = \tfrac{1}{N} \sum_{i=1}^{N} \big( \vh_{\text{icl}}(x_i, Z) - \vh_{\text{zs}}(x_i) \big)
    \end{equation*}
    where $N$ unlabeled train queries $\{x_i\}_{i=1}^N$ are used.
    
    \item \textbf{Linear Mapping \revise{(\oursmethod{})}:} Mapping from $\vh_{\text{zs}}$ to the target $\vh_{\text{icl}}-\vh_{\text{zs}}$ using a linear transformation, \ie $f(x,Z) = \mW \vh_{\text{zs}}(x)$, where $\mW \in \mathbb{R}^{d \times d}$.
    This has a closed-form solution when we formulate it as a ridge regression problem; see Section~\ref{subsec:ltv}.
    
    \item \textbf{2-layer MLP-based Mapping:} Mapping from $\vh_{\text{zs}}$ to the target $\vh_{\text{icl}}-\vh_{\text{zs}}$ using a 2-layer ReLU~\citep{agarap2018deep} network: $f(x,Z) = \mW_2 \, \sigma_{\text{ReLU}}(\mW_1 \vh_{\text{zs}}(x))$, where $\mW_1, \mW_2 \in \mathbb{R}^{d \times d}$ are learnable parameters and $\sigma_{\text{ReLU}}$ denotes the ReLU activation function.
    This is more expressive than the linear mapping, but requires iterative optimization.
    We train this network by minimizing the $\mathcal{L}_{\text{MSE}}$ loss using AdamW~\citep{loshchilov2017decoupled} with learning rate $10^{-3}$, cosine scheduler with warmup ratio 0.1, batch size 8, and 20 epochs over $N=256$ unlabeled train queries with $k=30$ demonstrations $Z$.
\end{itemize}

\input{table/mapping}
\input{figure/mlp_comparison}

As shown in Table~\ref{tab:mapping}, the constant mapping achieves only 55.72\% accuracy, performing significantly worse than our proposed linear mapping (75.21\% accuracy) by a margin of 19.49\% in accuracy.
This confirms that a query-dependent mapping is essential for approximating the effect of demonstrations, as a fixed hidden state vector cannot capture how this effect varies across queries.

On the other hand, the 2-layer MLP-based mapping improves accuracy to 77.74\%, indicating that more expressive mappings can yield additional gains within our framework.
This observation suggests that our optimization formulation in~\eqref{eq:opt_mse} naturally extends beyond the linear case.
In this paper, we focus on the linear mapping as the simplest instantiation that admits an efficient closed-form solution and preserves the training-free nature of ICL, and leave non-linear extensions that require iterative optimization for future work.
Below, we provide a preliminary cost-benefit analysis of such non-linear mappings when training is allowed, to characterize the trade-off between expressiveness and extraction cost.

\vspace{-1pt}
\paragraph{Further discussion on scaling of the non-linear mappings.}
To further explore how non-linear mappings scale with model capacity, we train deeper MLP-based mappings (4, 8, and 16 layers), each with skip connections every two layers to preserve the residual stream.
Fig.~\ref{fig:mlp_comparison} compares \oursmethod{} and these MLP-based variants in terms of both extraction (training for MLP) time cost and classification accuracy.
While deeper MLPs achieve higher absolute accuracy (up to 79.36\% with 16 layers), the gain comes with substantially increased extraction cost: a 16-layer MLP requires roughly $5.5\times$ the extraction time of \oursmethod{} for 4.15\% absolute accuracy improvement.

\subsection{Implementation Details on Baselines}
\label{app:baseline_methods}
In this section, we provide the details of each baseline along with our implementation.
For each baseline, we describe (1) the extraction phase, (2) the inference phase, and (3) the implementation details.
Note that these baselines typically require additional \emph{labeled} examples for layer or head selection (\eg Task Vector, Function Vector, and State Vector use a labeled validation set of 32 examples), whereas our method only uses \emph{unlabeled} queries.

\begin{itemize}
    \item \textbf{Task Vector}~\citep{hendel2023context}.
    In the extraction phase, task vectors are extracted from hidden states at the last token position.
    In the inference phase, the original hidden state of the model at a specific layer is replaced with the extracted vector (\ie $\vh_{\text{tv}} = \vv$).
    Following the original study, we select the optimal layer independently for each task using a validation set with 32 \emph{labeled} examples.
    \item \textbf{Function Vector}~\citep{todd2023function}.
    In the extraction phase, task vectors are extracted from the activations of attention heads that are selected based on their causal influence on increasing the probability of the correct label.
    In the inference phase, the task vector is added to the hidden state at a chosen injection layer.
    In our implementation, we follow the original paper and identify the top-10 attention heads based on their causal influence.
    We average their representations over 20 independent trials and select the optimal layer using a validation set with 32 labeled examples.
    Detailed implementation refers to the official repository at \url{https://github.com/ericwtodd/function_vectors}.
    \item \textbf{State Vector}~\citep{li2024context}.
    In the extraction phase, task vectors are formed by collecting attention activations at the separator tokens of each demonstration.
    The activations are gathered across the initial layers.
    In the inference phase, these stored activations are added to the model activations.
    We employ the inner optimization strategy from the original paper, which averages activations extracted from multiple demonstration examples.
    The optimal depth of the initial layers is selected for each task via a validation set with 32 labeled examples.
    The method is implemented using the official code at \url{https://github.com/HITsz-TMG/ICL-State-Vector}.
    \item \textbf{I2CL}~\citep{li2024implicit}.
    In the extraction phase, task vectors are extracted from the output activations of both attention and MLP modules at the last token position. These activations are averaged over each demonstration example.
    At inference time, a linear combination of these attention and MLP task vectors is additively injected into the residual stream at every layer.
    In our experiments, we use the initial coefficient values ($\lambda=0.1, \beta=1$) for the linear combination as suggested in the original paper.
    Code and detailed implementation instructions are available at \url{https://github.com/LzVv123456/I2CL}.
\end{itemize}

\subsection{Further Experiments Related to Table~\ref{tab:comparison_with_baseline} in Section~\ref{subsec:exp_results}}
\label{app:additional_exp}

\input{table/additional_comparison_with_baseline}

\paragraph{Extended results of Table~\ref{tab:comparison_with_baseline}.}
Table~\ref{tab:additional_with_baseline} provides the complete results across all five LLMs evaluated in this work. Consistent with the trend reported in the main body, \oursmethod{} achieves the highest average accuracy on all models, demonstrating that its effectiveness generalizes across different model families and scales.

\paragraph{Comparison with I2CL under Its Original Validation and Tuning Protocol.}
\label{app:optimized_i2cl}
\input{table/comparison_with_optimized_i2cl}
In the main experiments, we use the default I2CL scaling coefficient to compare task vector methods without gradient-based training. Here, we additionally evaluate I2CL using the validation and tuning protocol described in its original paper~\citep{li2024implicit}, while keeping $N=256$ and $\lambda=5$ fixed for \oursmethod{}. Table~\ref{tab:optimized_i2cl} reports the comparison across the five models and eight classification benchmarks. \oursmethod{} achieves higher average accuracy on Qwen-2.5-7B, Qwen-3-8B, and LLaMA-3.1-8B, whereas I2CL achieves higher average accuracy on LLaMA-2-7B and LLaMA-2-13B.

\input{table/shared_queries}
\paragraph{Comparison using the same training queries.}

In our main experiments, each task vector method uses a different number of training queries, following its original setup, as described in Section~\ref{subsec:exp_setup}.
One might wonder whether the performance advantage of \oursmethod{} persists when all task vector methods are given the same training queries. To examine this, we compare all methods using the same pool of 256 training queries on LLaMA-3.1-8B across the eight classification benchmarks. Importantly, \oursmethod{} does not use the query labels, whereas the other task vector methods use the corresponding labels to construct their task vectors. As shown in Table~\ref{tab:shared_queries}, \oursmethod{} still achieves the highest average accuracy among task vector methods when the training queries are shared.

\input{table/comparison_10_shot}
\input{table/comparison_50_shot}

\vspace{-1pt}
\paragraph{Effect of the number of demonstrations.}
To verify the robustness of \oursmethod{} to the number of demonstrations $k$, we also report results under $k=10$ (Table~\ref{tab:10shot}) and $k=50$ (Table~\ref{tab:50shot}) on LLaMA-3.1-8B.
In both settings, \oursmethod{} consistently achieves the highest average accuracy among task vector baselines across both settings, demonstrating its robustness to the number of demonstrations $k$.

\subsection{Implementation Details on Regression Experiments}
\label{app:regression}
We provide the detailed experimental setup for the regression experiments reported in Table~\ref{tab:regression}.
Following prior work in the in-context learning literature~\citep{garg2022can,yang2025task}, we evaluate \oursmethod{} on synthetic regression tasks.
For each task instance, the model is given 30 demonstration pairs $(\vx, \vy)$ and is asked to predict the continuous scalar output for a query input. We use a simple textual template where each demonstration is formatted as two lines: $x: [x_1, \ldots, x_m]$ followed by $y: y_i$, with demonstrations separated by blank lines.

Each input vector is sampled as $\vx \sim \mathcal{N}(0, \mI_m)$ with $m=6$.
We consider two regression settings: linear regression, where $\vy = \vw^\top \vx$ for $\vw \in \mathbb{R}^d$, and ReLU~\citep{agarap2018deep} regression, where $\vy = \boldsymbol{\alpha}^\top \mathrm{ReLU}(\mV \vx)$ for $\mV \in \mathbb{R}^{r \times d}$, $\boldsymbol{\alpha} \in \mathbb{R}^r$, and $r=100$. For linear regression, the parameter is sampled as $w \sim \mathcal{N}(0, I_d)$, following the isotropic Gaussian prior of~\citet{garg2022can}. For ReLU regression, we sample $V \sim \mathcal{N}(0, I)$ and $\alpha \sim \mathcal{N}(0, \frac{1}{r} I)$. The performance is measured by mean squared error (MSE).

Unlike classification, where each label is a single token from a discrete label set, the label $\vy$ here is a real-valued scalar that the LLM must generate as a multi-token sequence (digits and decimal point).
To extend \oursmethod{} to this multi-token generation setting, we apply the linear mapping $\mW^\star$ at every generation step.
At each step, the model takes the query and the tokens generated so far as input, and computes the zero-shot hidden state $\vh_{\text{zs}}$ of the last token.
We then add the task vector $\mW^\star \vh_{\text{zs}}$ to it before feeding the result to the LM head.
The sampled token is appended to the input for the next step.
Note that $\mW^\star$ is estimated once during extraction and reused across all generation steps, preserving the inference efficiency.

\subsection{Implementation Details on Cross-Model Transfer Experiments}
\label{app:transfer}

We provide the detailed experimental setup and method for the cross-model transfer experiments reported in Table~\ref{tab:ltv_transfer}.
The proposed \oursmethod{} method in Sec.~\ref{sec:method} is defined within a single model.
In practice, however, when the capacity of the model is not enough, ICL alone may not extract sufficient task information from the demonstrations, leading to limited quality of the task vector.

We thus consider transferring task vectors extracted from a larger model -- which captures richer ICL representations from the same demonstrations -- to a smaller model, in order to improve its performance without requiring additional demonstrations or larger inference cost.
This transfer is non-trivial because the two models have different hidden state dimensions and different LM heads, so the linear mapping $\mW^\star$ in Sec.~\ref{sec:method} cannot be directly reused.

\paragraph{Method.}
We extend \oursmethod{} to the cross-model setting with a single modification.
We replace the hidden-space regression in~\eqref{eq:ridge} with a logit-space regression.
The intuition is that, while hidden state spaces differ across models, two models that share the same tokenizer share the vocabulary as a common coordinate system, making the logit space a natural target for alignment.
To be specific, let $\vz_{\mathrm{icl}}^L \coloneqq \mW_{\mathrm{lm}}^L \vh_{\mathrm{icl}}^L$ denote the logit produced by the \revise{larger} model under ICL, and $\vz_{\mathrm{zs}}^S \coloneqq \mW_{\mathrm{lm}}^S \vh_{\mathrm{zs}}^S$ the logit produced by the \revise{smaller} model under zero-shot inference, where the superscripts $L$ and $S$ denote quantities from the \revise{larger and smaller} model, respectively.

We solve a ridge regression analogous to~\eqref{eq:ridge}, mapping $\vh_{\mathrm{zs}}^S$ to the logit-space difference $\vz_{\mathrm{icl}}^L - \vz_{\mathrm{zs}}^S$, which yields a closed-form solution $\tilde{\mW}^\star \in \mathbb{R}^{N_\mathcal{U} \times d_S}$, where $d_{S}$ is the hidden dimension of the \revise{smaller} model and $N_\mathcal{U}$ is the vocabulary size.
At inference, the corrected logit is $\vz_{\mathrm{zs}}^S(\vx_{\mathrm{test}}) + \tilde{\mW}^\star \vh_{\mathrm{zs}}^S(\vx_{\mathrm{test}})$.
When the \revise{source and target} models are identical, this reduces to the original formulation in~\eqref{eq:closed_form} up to the linear LM head transformation, ensuring that our extension preserves the design principle of \oursmethod{}.

\paragraph{Setup.}
We use Qwen2.5-72B as the \revise{larger} model and Qwen2.5-7B as the \revise{smaller} model, which share the same tokenizer.
We follow the eight classification benchmarks and prompt templates used in the main experiments in Sec.~\ref{sec:exp}.
To target the regime where the \revise{smaller} model cannot extract sufficient task information from demonstrations alone, we use $k=10$ demonstrations, compared to $k=30$ in the main experiments.
For DBPedia, we use $k=14$ since it has 14 classes.
We use $N=256$ unlabeled training queries and $\lambda=5.0$ for the ridge regression, following the default values in the main experiments.
We report the average accuracy over five independent runs.

\subsection{Additional Experiment on Generation Tasks}
\label{app:further_experiments}
\label{app:generation}

\input{table/generation}
We explore extending \oursmethod{} from single-token classification to multi-token generation. Recall that, in the classification setting, \oursmethod{} fits the linear mapping $\mW$ using the final-layer hidden states at the position immediately before the first output token is predicted. At inference, it adds $\mW\vh_{\mathrm{zs}}$ to the hidden state at this position before computing the predictive distribution. Thus, both regression and task vector injection concern the first output-token prediction.

For generation, the model must predict a sequence of tokens over the full vocabulary. Aligning its predictive distribution with ICL therefore requires considering subsequent decoding steps as well. To extend our proposed method to generation tasks, we consider two dimensions: the positions whose hidden states are used to fit $\mW$, and the decoding steps at which the task vector is injected.

We denote the resulting variants by $W_{\mathrm{fit},\mathrm{inject}}$, where $0$ indicates only the first output-token prediction and $\mathrm{all}$ indicates every output-token prediction. This gives four variants:
\begin{itemize}[leftmargin=*, itemsep=2pt]
    \item \textbf{$W_{0,0}$: Regression and injection only for the first output token.} We fit $\mW$ using hidden states at the position preceding the first output token and inject the task vector only before predicting that token. Subsequent tokens are generated without further task vector injection. This applies the original classification procedure to generation.
    \item \textbf{$W_{0,\mathrm{all}}$: Regression for the first output token and injection at every decoding step.} We fit $\mW$ using hidden states at the position preceding the first output token, but apply the mapping to the current hidden state before each output-token prediction. This extends only the injection positions.
    \item \textbf{$W_{\mathrm{all},0}$: Regression over all output positions and injection only for the first output token.} We fit $\mW$ using hidden states from all decoding positions of the ICL generation, but inject the task vector only before predicting the first output token. This extends only the regression positions.
    \item \textbf{$W_{\mathrm{all},\mathrm{all}}$: Regression over all output positions and injection at every decoding step.} We fit $\mW$ using hidden states from all decoding positions of the ICL generation and apply the mapping to the current hidden state before each output-token prediction. This extends both the regression and injection positions.
\end{itemize}
In all four variants, $\mW$ is obtained once through a single ridge regression. When injection is applied at every decoding step, the fitted mapping is reused to compute the task vector from the current hidden state. Table~\ref{tab:generation} compares these four choices.

We evaluate on LLaMA-3.1-8B using WMT14 English-to-French translation~\citep{bojar2014findings}, XSum summarization~\citep{narayan2018don}, and GSM8K mathematical reasoning~\citep{cobbe2021training} with 8-shot chain-of-thought prompting. We use 256 unlabeled training queries to fit $\mW$ and average the results over three runs with different demonstration sets. We report BLEU~\citep{papineni2002bleu} for translation, ROUGE-1~\citep{lin2004rouge} for summarization, and accuracy (\%) for mathematical reasoning.

As shown in Table~\ref{tab:generation}, fitting and injecting only at the first position improves translation and summarization over zero-shot inference. Mathematical reasoning instead benefits from fitting and injecting at all positions. Fitting only at the first position and injecting at every decoding step performs worse than zero-shot inference on all three tasks. Thus, the choice of regression and injection positions matters for generation, and none of the four variants matches ICL. Extending \oursmethod{} to longer generation remains a direction for future work.

\subsection{Limitation and Future Work}
\label{app:limitation}

We describe several limitations of our work along with corresponding directions for future work.
First, we adopt a linear mapping for \oursmethod{} and empirically validate its effectiveness across a wide range of benchmarks and models.
At the same time, our experiments in Appendix~\ref{app:design_choices} show that more expressive mappings can further improve downstream accuracy, suggesting clear potential beyond the linear case.
Therefore, designing more expressive mappings that retain the efficiency of a closed-form solution would thus be a promising direction for future work.

Second, while our work investigates \emph{why} a linear mapping is a reasonable choice and \emph{how} effective it is, the question of \emph{when} the linear approximation is effective remains unexplored.
A deeper theoretical characterization of these conditions would not only strengthen the foundation of our work, but also guide the design of mappings tailored to different settings.

Third, our experiments primarily focus on tasks with relatively short outputs, such as single-token classification and scalar regression.
Extending \oursmethod{} to inference tasks that generate longer and more diverse token sequences -- such as multi-token reasoning, multi-turn conversation, or summarization -- would broaden its applicability and is a promising direction for future work.
We provide the preliminary generation experiments in Appendix~\ref{app:generation}.

Finally, although we demonstrate in Sec.~\ref{subsec:exp_results} that \oursmethod{} enables task vectors to be transferred from a larger model to a smaller one, this transfer currently relies on the assumption that the source and target models share the same tokenizer, which serves as a common coordinate system for the logit-space regression.
Extending \oursmethod{} to enable transfer across models with different tokenizers would substantially broaden its applicability and is another valuable direction for future work.

%% file: table/benchmark_details.tex
    \centering
    \small
    \caption{Overview of the benchmark datasets used in our experiments, including their prompting templates, label sets, and licenses. \{Sentence\} and \{Label\} are placeholders for input query $x$ and label $y$, respectively. All datasets are publicly available and used under their original licenses or standard research-use terms.}
    \label{tab:benchmark_details}
    \resizebox{1.0\linewidth}{!}{
    \begin{tabular}{l m{3cm} m{5.8cm} m{3cm}}
    \toprule
    \textbf{Dataset} & \textbf{Prompt Template} & \textbf{Labels} & \textbf{License} \\
    \midrule

    AGNews     
        & News: \{Sentence\} \newline Type: \{Label\}
        & World / Sports / Business / Technology
        & Publicly available for research purposes\\

    \midrule

    DBPedia    
        & Input: \{Sentence\} \newline Label: \{Label\}
        & company / school / artist / athlete / politics / transportation / building / nature / village / animal / plant / album / film / book
        & CC BY-SA 3.0\\

    \midrule

    HateSpeech18 
        & Text: \{Sentence\} \newline Label: \{Label\}
        & neutral / hate
        & CC BY-SA 3.0\\

    \midrule

    MR        
        & Review: \{Sentence\} \newline Sentiment: \{Label\}
        & negative / positive
        & Publicly available for research purposes\\

    \midrule

    SST-2     
        & Review: \{Sentence\} \newline Sentiment: \{Label\}
        & negative / positive
        & CC BY 4.0 \\

    \midrule

    SST-5     
        & Sentence: \{Sentence\} \newline Sentiment: \{Label\}
        & terrible / negative / neutral / positive / great
        & CC BY 4.0 \\

    \midrule

    SUBJ      
        & Sentence: \{Sentence\} \newline Label: \{Label\}
        & objective / subjective
        & Publicly available for research purposes \\

    \midrule

    TREC      
        & Question: \{Sentence\} \newline Answer Type: \{Label\}
        & Abbreviation / Entity / Person / Location / Number
        & Publicly available for research purposes \\

    \bottomrule
    \end{tabular}
    }
\vspace{5pt}

%% file: table/mapping.tex
\begin{table}[t]
\centering
\caption{Comparison of design choices for the task vector $f(x,Z)$ that maps $\vh_{\text{zs}}(x)$ to the estimated hidden state difference $\vh_{\text{icl}} - \vh_{\text{zs}}$.
We compare three ways of designing such mapping: the \textit{constant mapping} uses a fixed vector $f(x,Z)=\vc$, the \textit{linear mapping} which uses a linear transformation $f(x,Z)=\mW \vh_{\text{zs}}(x)$, and the \textit{2-layer MLP-based mapping} which uses a ReLU network $f(x,Z)=\mW_2 \sigma_{\text{ReLU}}(\mW_1 \vh_{\text{zs}}(x))$. Note that the \textit{linear mapping} is used in our proposed LTV method in Sec.~\ref{sec:method}. 
We report the accuracy averaged  over 8 classification benchmarks used in Sec.~\ref{sec:exp}\revise{ and} our proposed \oursmetric{} metric.
Experiments are conducted on LLaMA-3.1-8B~\citep{dubey2024llama}.}
\setlength{\tabcolsep}{15pt}
\resizebox{\linewidth}{!}
{
\begin{tabular}{llcc}
\toprule
Design Choices & $f(x,Z)$ & Avg. Acc. $\uparrow$ & \oursmetric{} $\downarrow$ \\
\midrule
Constant Mapping  & $\vc$ & 55.72 & 0.694 \\
\cellcolor{gray!15}Linear Mapping (\oursmethod{}) & \cellcolor{gray!15}$\mW \vh_{\text{zs}}(x)$ & \cellcolor{gray!15}75.21 & \cellcolor{gray!15}0.147 \\
2-layer MLP based Mapping & $\mW_2 \sigma_{\text{ReLU}}(\mW_1 \vh_{\text{zs}}(x))$ & 77.74 & 0.132 \\
\bottomrule
\end{tabular}
}
\label{tab:mapping}
\end{table}

%% file: figure/mlp_comparison.tex
\begin{figure*}[t]
    \vspace{5pt}
    \centering
\includegraphics[width=0.85\textwidth]{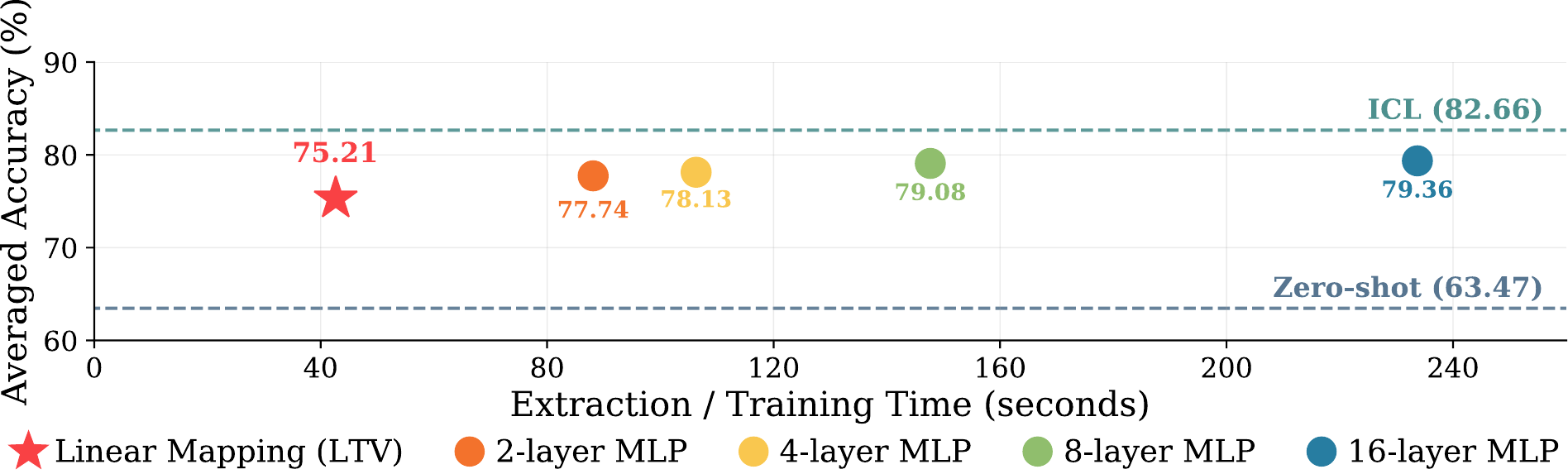}
    \caption{
    \revise{Extraction (training for MLP) time} versus downstream accuracy of \oursmethod{} and MLP-based mappings with varying depth.
    All methods are tested on LLaMA-3.1-8B, and accuracy is averaged over eight classification benchmarks used in the main paper in Sec.~\ref{sec:exp}.
    }

\label{fig:mlp_comparison}
\end{figure*}

%% file: table/additional_comparison_with_baseline.tex
\begin{table*}[t]
    \LARGE
    \centering
    \caption{
Complete results for Table~\ref{tab:comparison_with_baseline} in the main paper. We report the classification accuracy (\%) of LTV and four baseline task vector methods on eight benchmarks across all five LLMs evaluated in this work.
The accuracy of Zero-shot and ICL inference is provided as reference inference modes.
\oursmethod{} consistently achieves the highest average accuracy (Avg.) across all models.
}
    \resizebox{1.0\textwidth}{!}{
    {\setlength{\tabcolsep}{6pt}
    \begin{tabular}{l l| c c c c c c c c| c}
        \toprule
        Models & Methods
        & AGNews & DBPedia & HateSpeech18 & MR & SST-2 & SST-5 & Subj & TREC & Avg. \\
        \midrule

        \multirow{7}{*}{\shortstack{\textit{Qwen-2.5-7B}\\{\large\citep{qwen2024qwen2}}}}
        & Zero-shot
          & 73.80 & 74.80 & 58.40 & 68.00 & 54.80 & 20.20 & 48.80 & 66.40 & 58.15 \\
        & ICL
          & 85.40 {\large ($\pm$1.6)}
          & 97.32 {\large ($\pm$0.3)}
          & 81.76 {\large ($\pm$1.3)}
          & 93.44 {\large ($\pm$0.5)}
          & 94.40 {\large ($\pm$0.3)}
          & 48.56 {\large ($\pm$2.2)}
          & 89.68 {\large ($\pm$1.9)}
          & 87.32 {\large ($\pm$2.6)}
          & 84.74 \\
          \cmidrule(lr){2-11}
        & Function Vector{\large~\citep{todd2023function}}
          & 73.24 {\large ($\pm$0.4)}
          & 75.56 {\large ($\pm$0.5)}
          & 58.64 {\large ($\pm$0.5)}
          & 58.96 {\large ($\pm$1.2)}
          & 58.48 {\large ($\pm$1.8)}
          & 20.20 {\large ($\pm$0.0)}
          & 58.32 {\large ($\pm$3.5)}
          & 68.08 {\large ($\pm$0.7)}
          & 58.94 \\
        & Task Vector{\large~\citep{hendel2023context}}
          & 74.76 {\large ($\pm$2.1)}
          & \underline{78.52} {\large ($\pm$1.8)}
          & \underline{63.92} {\large ($\pm$3.3)}
          & \underline{87.28} {\large ($\pm$2.9)}
          & \textbf{93.76} {\large ($\pm$1.2)}
          & 23.52 {\large ($\pm$4.7)}
          & 54.96 {\large ($\pm$6.9)}
          & \underline{74.12} {\large ($\pm$3.2)}
          & 68.86 \\
        & State Vector{\large~\citep{li2024context}}
          & \textbf{78.52} {\large ($\pm$0.3)}
          & 76.88 {\large ($\pm$0.6)}
          & 57.84 {\large ($\pm$1.3)}
          & 85.84 {\large ($\pm$2.5)}
          & 86.64 {\large ($\pm$3.6)}
          & \underline{30.60} {\large ($\pm$0.2)}
          & \underline{70.12} {\large ($\pm$2.7)}
          & 71.64 {\large ($\pm$6.8)}
          & \underline{69.76} \\
        & I2CL{\large~\citep{li2024implicit}}
          & 75.60 {\large ($\pm$2.3)}
          & 75.68 {\large ($\pm$0.4)}
          & 62.80 {\large ($\pm$0.7)}
          & 79.24 {\large ($\pm$0.4)}
          & 87.36 {\large ($\pm$0.3)}
          & 22.96 {\large ($\pm$0.3)}
          & 42.72 {\large ($\pm$0.3)}
          & 56.68 {\large ($\pm$1.8)}
          & 62.88 \\
        & \bg{\textbf{\oursmethod{} (Ours)}}
          & \bg{\underline{77.32} {\large ($\pm$6.7)}}
          & \bg{\textbf{90.68} {\large ($\pm$2.2)}}
          & \bg{\textbf{75.24} {\large ($\pm$1.2)}}
          & \bg{\textbf{90.56} {\large ($\pm$1.3)}}
          & \bg{\underline{91.76} {\large ($\pm$1.7)}}
          & \bg{\textbf{31.32} {\large ($\pm$4.2)}}
          & \bg{\textbf{80.68} {\large ($\pm$5.6)}}
          & \bg{\textbf{83.40} {\large ($\pm$6.2)}}
          & \bg{\textbf{77.62}} \\
        \midrule

        \multirow{7}{*}{\shortstack{\textit{Qwen-3-8B}\\{\large\citep{yang2025qwen3}}}}
        & Zero-shot
          & 57.40 & 74.60 & 71.60 & 92.40 & 92.20 & 34.80 & 66.40 & 68.20 & 69.70 \\
        & ICL
          & 86.56 {\large ($\pm$2.3)}
          & 97.31 {\large ($\pm$0.6)}
          & 83.44 {\large ($\pm$1.0)}
          & 93.28 {\large ($\pm$0.5)}
          & 94.48 {\large ($\pm$0.5)}
          & 53.08 {\large ($\pm$2.3)}
          & 93.36 {\large ($\pm$0.7)}
          & 85.24 {\large ($\pm$4.8)}
          & 85.84 \\
          \cmidrule(lr){2-11}
        & Function Vector{\large~\citep{todd2023function}}
          & 61.68 {\large ($\pm$0.8)}
          & 77.74 {\large ($\pm$0.3)}
          & 71.56 {\large ($\pm$0.4)}
          & 90.24 {\large ($\pm$1.0)}
          & 91.36 {\large ($\pm$0.2)}
          & 34.88 {\large ($\pm$0.1)}
          & 67.20 {\large ($\pm$2.2)}
          & 75.56 {\large ($\pm$0.7)}
          & 71.28 \\
        & Task Vector{\large~\citep{hendel2023context}}
          & \underline{79.24} {\large ($\pm$7.0)}
          & 81.49 {\large ($\pm$1.4)}
          & \underline{74.20} {\large ($\pm$0.9)}
          & 92.36 {\large ($\pm$0.2)}
          & \underline{93.16} {\large ($\pm$0.6)}
          & 34.08 {\large ($\pm$0.4)}
          & \underline{76.00} {\large ($\pm$2.5)}
          & 73.56 {\large ($\pm$7.9)}
          & 75.51 \\
        & State Vector{\large~\citep{li2024context}}
          & \textbf{81.24} {\large ($\pm$4.2)}
          & \underline{84.40} {\large ($\pm$0.9)}
          & 69.92 {\large ($\pm$2.0)}
          & \textbf{92.44} {\large ($\pm$0.1)}
          & 92.24 {\large ($\pm$0.4)}
          & 33.92 {\large ($\pm$1.1)}
          & 75.56 {\large ($\pm$6.8)}
          & 78.32 {\large ($\pm$2.6)}
          & \underline{76.01} \\
        & I2CL{\large~\citep{li2024implicit}}
          & 62.92 {\large ($\pm$0.2)}
          & 74.51 {\large ($\pm$0.3)}
          & 70.12 {\large ($\pm$0.6)}
          & \underline{92.40} {\large ($\pm$0.1)}
          & \textbf{93.24} {\large ($\pm$0.2)}
          & \underline{34.92} {\large ($\pm$0.1)}
          & 59.76 {\large ($\pm$0.3)}
          & \underline{78.36} {\large ($\pm$0.9)}
          & 70.78 \\
        & \bg{\textbf{\oursmethod{} (Ours)}}
          & \bg{77.80 {\large ($\pm$2.1)}}
          & \bg{\textbf{94.20} {\large ($\pm$1.7)}}
          & \bg{\textbf{76.04} {\large ($\pm$3.4)}}
          & \bg{86.72 {\large ($\pm$6.3)}}
          & \bg{90.40 {\large ($\pm$2.2)}}
          & \bg{\textbf{39.40} {\large ($\pm$7.1)}}
          & \bg{\textbf{90.00} {\large ($\pm$2.2)}}
          & \bg{\textbf{81.96} {\large ($\pm$4.6)}}
          & \bg{\textbf{79.57}} \\
        \midrule

        \multirow{7}{*}{\shortstack{\textit{LLaMA-2-7B}\\{\large\citep{touvron2023llama}}}}
        & Zero-shot
          & 72.40 & 73.60 & 54.00 & 73.00 & 80.40 & 28.40 & 51.60 & 50.00 & 60.43 \\
        & ICL
          & 84.48 {\large ($\pm$5.1)}
          & 94.44 {\large ($\pm$2.7)}
          & 64.60 {\large ($\pm$9.0)}
          & 93.72 {\large ($\pm$0.6)}
          & 93.52 {\large ($\pm$1.2)}
          & 42.88 {\large ($\pm$3.3)}
          & 54.48 {\large ($\pm$6.5)}
          & 79.36 {\large ($\pm$4.0)}
          & 75.94 \\
          \cmidrule(lr){2-11}
        & Function Vector{\large~\citep{todd2023function}}
          & 71.64 {\large ($\pm$0.7)}
          & 73.88 {\large ($\pm$1.2)}
          & 55.80 {\large ($\pm$0.6)}
          & 74.72 {\large ($\pm$0.9)}
          & 78.60 {\large ($\pm$0.1)}
          & 27.84 {\large ($\pm$0.8)}
          & 50.04 {\large ($\pm$0.2)}
          & 59.28 {\large ($\pm$4.2)}
          & 61.48 \\
        & Task Vector{\large~\citep{hendel2023context}}
          & \underline{76.68} {\large ($\pm$2.6)}
          & 78.68 {\large ($\pm$0.6)}
          & \textbf{71.28} {\large ($\pm$2.8)}
          & 78.64 {\large ($\pm$6.8)}
          & 81.60 {\large ($\pm$5.6)}
          & 32.24 {\large ($\pm$3.6)}
          & 48.72 {\large ($\pm$2.6)}
          & 58.84 {\large ($\pm$6.8)}
          & 65.84 \\
        & State Vector{\large~\citep{li2024context}}
          & 74.36 {\large ($\pm$8.4)}
          & \textbf{90.36} {\large ($\pm$1.0)}
          & \underline{64.12} {\large ($\pm$3.7)}
          & \underline{88.48} {\large ($\pm$3.8)}
          & \underline{82.56} {\large ($\pm$7.2)}
          & 33.68 {\large ($\pm$8.4)}
          & 50.20 {\large ($\pm$3.7)}
          & \underline{62.20} {\large ($\pm$7.5)}
          & \underline{68.25} \\
        & I2CL{\large~\citep{li2024implicit}}
          & 72.96 {\large ($\pm$1.1)}
          & 76.00 {\large ($\pm$0.1)}
          & 51.48 {\large ($\pm$0.4)}
          & 74.16 {\large ($\pm$1.0)}
          & 80.16 {\large ($\pm$0.1)}
          & \underline{34.72} {\large ($\pm$0.6)}
          & \textbf{51.16} {\large ($\pm$0.1)}
          & 60.24 {\large ($\pm$0.5)}
          & 62.61 \\
        & \bg{\textbf{\oursmethod{} (Ours)}}
          & \bg{\textbf{83.96} {\large ($\pm$4.4)}}
          & \bg{\underline{88.84} {\large ($\pm$1.6)}}
          & \bg{55.72 {\large ($\pm$5.2)}}
          & \bg{\textbf{91.20} {\large ($\pm$2.5)}}
          & \bg{\textbf{90.12} {\large ($\pm$0.9)}}
          & \bg{\textbf{38.68} {\large ($\pm$2.0)}}
          & \bg{\underline{50.32} {\large ($\pm$0.3)}}
          & \bg{\textbf{70.68} {\large ($\pm$5.4)}}
          & \bg{\textbf{71.19}} \\
        \midrule

        \multirow{7}{*}{\shortstack{\textit{LLaMA-2-13B}\\{\large\citep{touvron2023llama}}}}
        & Zero-shot
          & 73.80 & 76.20 & 55.20 & 61.20 & 66.60 & 20.80 & 49.80 & 69.60 & 59.15 \\
        & ICL
          & 88.04 {\large ($\pm$1.7)}
          & 96.88 {\large ($\pm$1.1)}
          & 77.40 {\large ($\pm$2.2)}
          & 94.44 {\large ($\pm$0.5)}
          & 94.80 {\large ($\pm$0.7)}
          & 46.52 {\large ($\pm$2.7)}
          & 82.28 {\large ($\pm$7.1)}
          & 83.36 {\large ($\pm$5.2)}
          & 82.97 \\
          \cmidrule(lr){2-11}
        & Function Vector{\large~\citep{todd2023function}}
          & 74.00 {\large ($\pm$1.2)}
          & 76.40 {\large ($\pm$7.5)}
          & 54.08 {\large ($\pm$0.4)}
          & 61.84 {\large ($\pm$3.4)}
          & 72.44 {\large ($\pm$6.0)}
          & 21.40 {\large ($\pm$1.2)}
          & 50.00 {\large ($\pm$0.0)}
          & 69.96 {\large ($\pm$1.1)}
          & 60.02 \\
        & Task Vector{\large~\citep{hendel2023context}}
          & 79.84 {\large ($\pm$2.9)}
          & 80.72 {\large ($\pm$1.9)}
          & \underline{71.48} {\large ($\pm$6.7)}
          & 83.16 {\large ($\pm$0.8)}
          & 82.80 {\large ($\pm$4.2)}
          & 33.64 {\large ($\pm$2.1)}
          & \underline{49.72} {\large ($\pm$0.1)}
          & \underline{73.84} {\large ($\pm$5.3)}
          & 69.40 \\
        & State Vector{\large~\citep{li2024context}}
          & \underline{84.64} {\large ($\pm$4.0)}
          & \underline{89.48} {\large ($\pm$4.9)}
          & 58.80 {\large ($\pm$8.8)}
          & \underline{89.20} {\large ($\pm$2.8)}
          & \underline{87.20} {\large ($\pm$3.6)}
          & \underline{36.08} {\large ($\pm$3.0)}
          & 49.40 {\large ($\pm$0.6)}
          & 66.96 {\large ($\pm$1.0)}
          & \underline{70.22} \\
        & I2CL{\large~\citep{li2024implicit}}
          & 78.88 {\large ($\pm$0.3)}
          & 79.00 {\large ($\pm$0.5)}
          & 54.48 {\large ($\pm$0.1)}
          & 60.68 {\large ($\pm$0.3)}
          & 64.80 {\large ($\pm$0.4)}
          & 25.96 {\large ($\pm$0.1)}
          & 50.00 {\large ($\pm$0.0)}
          & 69.12 {\large ($\pm$0.4)}
          & 60.37 \\
        & \bg{\textbf{\oursmethod{} (Ours)}}
          & \bg{\textbf{86.68} {\large ($\pm$1.7)}}
          & \bg{\textbf{93.20} {\large ($\pm$0.5)}}
          & \bg{\textbf{72.08} {\large ($\pm$2.2)}}
          & \bg{\textbf{90.20} {\large ($\pm$2.0)}}
          & \bg{\textbf{88.96} {\large ($\pm$2.1)}}
          & \bg{\textbf{41.88} {\large ($\pm$2.2)}}
          & \bg{\textbf{78.76} {\large ($\pm$2.7)}}
          & \bg{\textbf{83.92} {\large ($\pm$6.3)}}
          & \bg{\textbf{79.46}} \\
        \midrule

        \multirow{7}{*}{\shortstack{\textit{LLaMA-3.1-8B}\\{\large\citep{dubey2024llama}}}}
        & Zero-shot
          & 75.00 & 69.00 & 60.60 & 82.20 & 86.80 & 25.60 & 59.20 & 49.40 & 63.48 \\
        & ICL
          & 87.16 {\large ($\pm$1.2)}
          & 97.68 {\large ($\pm$0.8)}
          & 74.32 {\large ($\pm$8.1)}
          & 94.52 {\large ($\pm$0.4)}
          & 94.20 {\large ($\pm$1.1)}
          & 48.32 {\large ($\pm$1.1)}
          & 85.96 {\large ($\pm$5.3)}
          & 79.08 {\large ($\pm$7.3)}
          & 82.66 \\
          \cmidrule(lr){2-11}
        & Function Vector{\large~\citep{todd2023function}}
          & 76.16 {\large ($\pm$0.1)}
          & 69.44 {\large ($\pm$1.0)}
          & 63.00 {\large ($\pm$0.2)}
          & 83.24 {\large ($\pm$0.3)}
          & 87.32 {\large ($\pm$0.6)}
          & 26.20 {\large ($\pm$1.3)}
          & 59.52 {\large ($\pm$0.2)}
          & \underline{69.12} {\large ($\pm$0.6)}
          & 66.75 \\
        & Task Vector{\large~\citep{hendel2023context}}
          & \underline{81.36} {\large ($\pm$3.7)}
          & \underline{83.24} {\large ($\pm$1.5)}
          & \underline{67.64} {\large ($\pm$2.0)}
          & 83.48 {\large ($\pm$3.5)}
          & 87.88 {\large ($\pm$0.3)}
          & 34.76 {\large ($\pm$1.9)}
          & 61.12 {\large ($\pm$1.4)}
          & 66.48 {\large ($\pm$1.0)}
          & \underline{70.75} \\
        & State Vector{\large~\citep{li2024context}}
          & 80.28 {\large ($\pm$4.3)}
          & 80.80 {\large ($\pm$1.5)}
          & 65.40 {\large ($\pm$1.3)}
          & \underline{85.96} {\large ($\pm$4.9)}
          & 84.12 {\large ($\pm$0.6)}
          & \underline{36.60} {\large ($\pm$1.6)}
          & \underline{62.52} {\large ($\pm$3.5)}
          & 67.28 {\large ($\pm$1.9)}
          & 70.37 \\
        & I2CL{\large~\citep{li2024implicit}}
          & 76.76 {\large ($\pm$0.4)}
          & 72.56 {\large ($\pm$0.4)}
          & 62.24 {\large ($\pm$0.5)}
          & 85.24 {\large ($\pm$0.3)}
          & \textbf{90.80} {\large ($\pm$0.2)}
          & 32.48 {\large ($\pm$0.4)}
          & 62.28 {\large ($\pm$0.1)}
          & 49.00 {\large ($\pm$0.5)}
          & 66.42 \\
        & \bg{\textbf{\oursmethod{} (Ours)}}
          & \bg{\textbf{82.84} {\large ($\pm$2.5)}}
          & \bg{\textbf{93.36} {\large ($\pm$1.4)}}
          & \bg{\textbf{70.44} {\large ($\pm$6.0)}}
          & \bg{\textbf{88.68} {\large ($\pm$1.0)}}
          & \bg{\underline{90.08} {\large ($\pm$1.2)}}
          & \bg{\textbf{38.20} {\large ($\pm$3.4)}}
          & \bg{\textbf{67.16} {\large ($\pm$11.4)}}
          & \bg{\textbf{72.88} {\large ($\pm$7.2)}}
          & \bg{\textbf{75.46}} \\
        \bottomrule
    \end{tabular}
    }
    }
    \label{tab:additional_with_baseline}
\end{table*}

%% file: table/comparison_with_optimized_i2cl.tex
\begin{table*}[t]
    \LARGE
    \centering
    \caption{Classification accuracy (\%) of \oursmethod{} and I2CL~\citep{li2024implicit} evaluated using its original validation and tuning protocol, tested on eight benchmarks across five models. Avg. reports the average accuracy across benchmarks.}
    \label{tab:optimized_i2cl}
    \resizebox{1.0\textwidth}{!}{
    {\setlength{\tabcolsep}{10pt}
    \begin{tabular}{l l| c c c c c c c c| c}
        \toprule
        Models & Methods & AGNews & DBPedia & HateSpeech18 & MR & SST-2 & SST-5 & Subj & TREC & Avg. \\
        \midrule
        \multirow{2}{*}{\shortstack{\textit{Qwen-2.5-7B}\\{\large\citep{qwen2024qwen2}}}}
        & I2CL (Optimized){\large~\citep{li2024implicit}}
          & \textbf{83.12} {\large ($\pm$2.9)}
          & 83.08 {\large ($\pm$1.5)}
          & 53.92 {\large ($\pm$3.4)}
          & 87.88 {\large ($\pm$1.2)}
          & 91.04 {\large ($\pm$1.7)}
          & \textbf{34.96} {\large ($\pm$4.1)}
          & 50.52 {\large ($\pm$1.1)}
          & 65.56 {\large ($\pm$9.5)}
          & 68.76 \\
        \cmidrule(lr){2-11}
        & \bg{\textbf{\oursmethod{} (Ours)}}
          & \bg{77.32 {\large ($\pm$6.7)}}
          & \bg{\textbf{90.68} {\large ($\pm$2.2)}}
          & \bg{\textbf{75.24} {\large ($\pm$1.2)}}
          & \bg{\textbf{90.56} {\large ($\pm$1.3)}}
          & \bg{\textbf{91.76} {\large ($\pm$1.7)}}
          & \bg{31.32 {\large ($\pm$4.2)}}
          & \bg{\textbf{80.68} {\large ($\pm$5.6)}}
          & \bg{\textbf{83.40} {\large ($\pm$6.2)}}
          & \bg{\textbf{77.62}} \\
        \midrule
        \multirow{2}{*}{\shortstack{\textit{Qwen-3-8B}\\{\large\citep{yang2025qwen3}}}}
        & I2CL (Optimized){\large~\citep{li2024implicit}}
          & \revise{\textbf{81.48} {\large ($\pm$2.9)}}
          & \revise{87.48 {\large ($\pm$0.3)}}
          & \revise{73.60 {\large ($\pm$2.4)}}
          & \revise{\textbf{88.68} {\large ($\pm$1.3)}}
          & \textbf{90.76} {\large ($\pm$1.6)}
          & \revise{\textbf{40.00} {\large ($\pm$2.3)}}
          & \revise{58.72 {\large ($\pm$8.6)}}
          & 71.44 {\large ($\pm$8.6)}
          & \revise{72.10} \\
        
        \cmidrule(lr){2-11}
        & \bg{\textbf{\oursmethod{} (Ours)}}
          & \bg{77.80 {\large ($\pm$2.1)}}
          & \bg{\textbf{94.20} {\large ($\pm$1.7)}}
          & \bg{\textbf{76.04} {\large ($\pm$3.4)}}
          & \bg{86.72 {\large ($\pm$6.3)}}
          & \bg{90.40 {\large ($\pm$2.2)}}
          & \bg{39.40 {\large ($\pm$7.1)}}
          & \bg{\textbf{90.00} {\large ($\pm$2.2)}}
          & \bg{\textbf{81.96} {\large ($\pm$4.6)}}
          & \bg{\textbf{79.57}} \\
        \midrule
        \multirow{2}{*}{\shortstack{\textit{LLaMA-2-7B}\\{\large\citep{touvron2023llama}}}}
        & I2CL (Optimized){\large~\citep{li2024implicit}}
          & 82.88 {\large ($\pm$3.3)}
          & \textbf{89.60} {\large ($\pm$2.4)}
          & \textbf{72.04} {\large ($\pm$2.4)}
          & 90.40 {\large ($\pm$1.7)}
          & 88.40 {\large ($\pm$1.7)}
          & \textbf{46.24} {\large ($\pm$2.2)}
          & \textbf{81.64} {\large ($\pm$3.1)}
          & 66.76 {\large ($\pm$8.6)}
          & \textbf{77.25} \\
        \cmidrule(lr){2-11}
        & \bg{\textbf{\oursmethod{} (Ours)}}
          & \bg{\textbf{83.96} {\large ($\pm$4.4)}}
          & \bg{88.84 {\large ($\pm$1.6)}}
          & \bg{55.72 {\large ($\pm$5.2)}}
          & \bg{\textbf{91.20} {\large ($\pm$2.5)}}
          & \bg{\textbf{90.12} {\large ($\pm$0.9)}}
          & \bg{38.68 {\large ($\pm$2.0)}}
          & \bg{50.32 {\large ($\pm$0.3)}}
          & \bg{\textbf{70.68} {\large ($\pm$5.4)}}
          & \bg{71.19} \\
        \midrule
        \multirow{2}{*}{\shortstack{\textit{LLaMA-2-13B}\\{\large\citep{touvron2023llama}}}}
        & I2CL (Optimized){\large~\citep{li2024implicit}}
          & 85.80 {\large ($\pm$4.2)}
          & 91.92 {\large ($\pm$0.8)}
          & \textbf{75.27} {\large ($\pm$3.4)}
          & \textbf{90.35} {\large ($\pm$1.6)}
          & \textbf{92.24} {\large ($\pm$1.1)}
          & \textbf{46.72} {\large ($\pm$1.7)}
          & \textbf{86.76} {\large ($\pm$2.4)}
          & 73.48 {\large ($\pm$9.0)}
          & \textbf{80.46} \\
        \cmidrule(lr){2-11}
        & \bg{\textbf{\oursmethod{} (Ours)}}
          & \bg{\textbf{86.68} {\large ($\pm$1.7)}}
          & \bg{\textbf{93.20} {\large ($\pm$0.5)}}
          & \bg{72.08 {\large ($\pm$2.2)}}
          & \bg{90.20 {\large ($\pm$2.0)}}
          & \bg{88.96 {\large ($\pm$2.1)}}
          & \bg{41.88 {\large ($\pm$2.2)}}
          & \bg{78.76 {\large ($\pm$2.7)}}
          & \bg{\textbf{83.92} {\large ($\pm$6.3)}}
          & \bg{79.46} \\
        \midrule
        \multirow{2}{*}{\shortstack{\textit{LLaMA-3.1-8B}\\{\large\citep{dubey2024llama}}}}
        & I2CL (Optimized){\large~\citep{li2024implicit}}
          & \revise{\textbf{84.60} {\large ($\pm$2.5)}}
          & \revise{84.60 {\large ($\pm$4.2)}}
          & 50.08 {\large ($\pm$0.2)}
          & \revise{87.52 {\large ($\pm$1.4)}}
          & \revise{87.92 {\large ($\pm$0.4)}}
          & \revise{\textbf{38.80} {\large ($\pm$4.8)}}
          & \revise{\textbf{70.80} {\large ($\pm$2.9)}}
          & \revise{70.28 {\large ($\pm$9.1)}}
          & \revise{71.83} \\
        
        \cmidrule(lr){2-11}
        & \bg{\textbf{\oursmethod{} (Ours)}}
          & \bg{82.84 {\large ($\pm$2.5)}}
          & \bg{\textbf{93.36} {\large ($\pm$1.4)}}
          & \bg{\textbf{70.44} {\large ($\pm$6.0)}}
          & \bg{\textbf{88.68} {\large ($\pm$1.0)}}
          & \bg{\textbf{90.08} {\large ($\pm$1.2)}}
          & \bg{38.20 {\large ($\pm$3.4)}}
          & \bg{67.16 {\large ($\pm$11.4)}}
          & \bg{\textbf{72.88} {\large ($\pm$7.2)}}
          & \bg{\textbf{75.46}} \\
        \bottomrule
    \end{tabular}
    }
    }
\end{table*}

%% file: table/shared_queries.tex
\begin{table*}[t]
    \LARGE
    \centering
    \caption{Classification accuracy (\%) on LLaMA-3.1-8B when all task vector methods use the same 256 training queries. \oursmethod{} does not use query labels, whereas the other task vector methods use the corresponding labels. Zero-shot and ICL are included for reference.}
    \label{tab:shared_queries}
    \resizebox{1.0\textwidth}{!}{
    {\setlength{\tabcolsep}{10pt}
    \begin{tabular}{l| c c c c c c c c| c}
        \toprule
        Methods & AGNews & DBPedia & HateSpeech18 & MR & SST-2 & SST-5 & Subj & TREC & Avg. \\
        \midrule
        Zero-shot & 75.00 & 69.00 & 60.60 & 82.20 & 86.80 & 25.60 & 59.20 & 49.40 & 63.47 \\
        ICL & 87.16 {\large ($\pm$1.1)} & 97.68 {\large ($\pm$0.7)} & 74.32 {\large ($\pm$7.2)} & 94.52 {\large ($\pm$0.4)} & 94.20 {\large ($\pm$1.0)} & 48.32 {\large ($\pm$1.0)} & 85.96 {\large ($\pm$4.7)} & 79.08 {\large ($\pm$6.5)} & 82.66 \\
        \midrule
        Function Vector{\large~\citep{todd2023function}} & 76.48 {\large ($\pm$0.8)} & 70.20 {\large ($\pm$0.3)} & 64.76 {\large ($\pm$0.9)} & 82.76 {\large ($\pm$0.5)} & 87.44 {\large ($\pm$0.4)} & 25.60 {\large ($\pm$0.9)} & 59.80 {\large ($\pm$0.4)} & 68.92 {\large ($\pm$0.7)} & 67.00 \\
        Task Vector{\large~\citep{hendel2023context}} & 82.84 {\large ($\pm$3.4)} & 83.76 {\large ($\pm$0.7)} & \textbf{71.16} {\large ($\pm$2.2)} & 87.00 {\large ($\pm$3.0)} & 88.32 {\large ($\pm$1.5)} & 35.36 {\large ($\pm$2.1)} & 61.20 {\large ($\pm$0.5)} & 68.92 {\large ($\pm$0.9)} & 72.32 \\
        State Vector{\large~\citep{li2024context}} & \textbf{82.96} {\large ($\pm$2.1)} & 80.96 {\large ($\pm$1.5)} & 67.40 {\large ($\pm$2.4)} & \textbf{90.72} {\large ($\pm$0.7)} & 85.84 {\large ($\pm$2.1)} & 36.44 {\large ($\pm$1.4)} & \textbf{68.68} {\large ($\pm$0.3)} & 67.40 {\large ($\pm$1.8)} & 72.55 \\
        I2CL{\large~\citep{li2024implicit}} & 76.76 {\large ($\pm$0.3)} & 72.60 {\large ($\pm$0.3)} & 62.24 {\large ($\pm$0.4)} & 85.24 {\large ($\pm$0.2)} & \textbf{90.80} {\large ($\pm$0.2)} & 32.48 {\large ($\pm$0.3)} & 62.28 {\large ($\pm$0.1)} & 49.00 {\large ($\pm$0.5)} & 66.42 \\
        \bg{\textbf{\oursmethod{} (Ours)}} & \bg{82.84 {\large ($\pm$2.5)}} & \bg{\textbf{93.36} {\large ($\pm$1.4)}} & \bg{70.44 {\large ($\pm$6.0)}} & \bg{88.68 {\large ($\pm$1.0)}} & \bg{90.08 {\large ($\pm$1.2)}} & \bg{\textbf{38.20} {\large ($\pm$3.4)}} & \bg{67.16 {\large ($\pm$11.4)}} & \bg{\textbf{72.88} {\large ($\pm$7.2)}} & \bg{\textbf{75.46}} \\
        \bottomrule
    \end{tabular}
    }
    }
\end{table*}

%% file: table/comparison_10_shot.tex
\begin{table*}[t]
\centering
    \caption{
    Extended results of Table~\ref{tab:comparison_with_baseline} on LLaMA-3.1-8B under the number of demonstrations $k=10$ setting.
    The experimental setup is otherwise identical to Table~\ref{tab:comparison_with_baseline}.
}
\label{tab:10shot}
\resizebox{\textwidth}{!}{
{\setlength{\tabcolsep}{12pt}
\begin{tabular}{l| c c c c c c c c| c}
\toprule
Methods
& AGNews & DBpedia & HateSpeech18 & MR & SST-2 & SST-5 & Subj & TREC & Avg. \\
    \midrule
    
        Zero-shot
        & 75.00 & 69.00 & 60.60 & 82.20 & 86.80 & 25.60 & 59.20 & 49.40 & 63.48 \\
        
        ICL
        & 87.52 & 96.44 & 73.00 & 94.64 & 94.32 & 46.20 & 81.84 & 80.20 & 81.77 \\
    
    \midrule
        Function Vector{\small~\citep{todd2023function}}
        & 75.72 & 69.80 & 63.80 & 82.04 & 86.80 & 25.56 & 59.56 & \underline{67.40} & 66.34 \\
        
        Task Vector{\small~\citep{hendel2023context}}
        & 80.16 & \underline{81.84} & \textbf{68.68} & 84.24 & 87.08 & 31.68 & 59.64 & 66.32 & 69.95 \\
        
        State Vector{\small~\citep{li2024context}}
        & \textbf{84.28} & 80.80 & \underline{67.24} & \textbf{90.80} & 79.16 & \underline{33.48} & \underline{67.20} & 60.64 & \underline{70.45} \\
        
        I2CL{\small~\citep{li2024implicit}}
        & 76.64 & 72.44 & 62.00 & 85.20 & \textbf{90.64} & 32.12 & 61.52 & 48.48 & 66.13 \\
        
        \rowcolor{gray!15}\oursmethod{} (Ours)
        & \underline{82.20}
        & \textbf{90.04}
        & 64.44
        & \underline{89.44}
        & \underline{88.92}
        & \textbf{41.60}
        & \textbf{69.60}
        & \textbf{72.92}
        & \textbf{74.89} \\
    
\bottomrule
\end{tabular}
}
}
\end{table*}

%% file: table/comparison_50_shot.tex
\begin{table*}[t]
\vspace{5pt}
\centering
    \caption{
    Extended results of Table~\ref{tab:comparison_with_baseline} on LLaMA-3.1-8B under the number of demonstrations $k=50$ setting.
    The experimental setup is otherwise identical to Table~\ref{tab:comparison_with_baseline}.
    }
\label{tab:50shot}
\resizebox{\textwidth}{!}{
{\setlength{\tabcolsep}{12pt}
\begin{tabular}{l| c c c c c c c c| c}
\toprule
Methods
& AGNews & DBPedia & HateSpeech18 & MR & SST-2 & SST-5 & Subj & TREC & Avg. \\
\midrule
    Zero-shot
    & 75.00 & 69.00 & 60.60 & 82.20 & 86.80 & 25.60 & 59.20 & 49.40 & 63.48 \\
    
    ICL
    & 87.84 & 97.24 & 78.92 & 94.60 & 94.52 & 50.28 & 90.68 & 85.68 & 84.97 \\
    
    \midrule
    
    Function Vector{\small~\citep{todd2023function}}
    & 76.56 & 69.92 & 63.56 & 83.00 & 87.48 & 26.16 & 59.48 & 67.08 & 66.66 \\
    
    Task Vector{\small~\citep{hendel2023context}}
    & 77.60 & \underline{82.36} & \underline{65.60} & 81.72 & 88.24 & 32.24 & 59.44 & 67.60 & 69.35 \\
    
    State Vector{\small~\citep{li2024context}}
    & \underline{83.56} & 81.12 & \underline{65.60} & 84.28 & 79.04 & \underline{33.64} & 61.48 & \underline{67.72} & \underline{69.56} \\
    
    I2CL{\small~\citep{li2024implicit}}
    & 77.16 & 72.84 & 62.32 & \underline{85.52} & \textbf{90.64} & 32.44 & \underline{61.96} & 49.40 & 66.53 \\
    
    \rowcolor{gray!15}\oursmethod{} (Ours)
    & \textbf{83.92}
    & \textbf{92.44}
    & \textbf{74.52}
    & \textbf{87.92}
    & \underline{89.80}
    & \textbf{41.80}
    & \textbf{69.88}
    & \textbf{76.72}
    & \textbf{77.12} \\
\bottomrule
\end{tabular}
}
}
\end{table*}

%% file: table/generation.tex
\begin{table*}[t]
    \LARGE
    \centering
    \caption{Performance on generation tasks with LLaMA-3.1-8B. Results are averaged over three runs with different demonstration sets.}
    \label{tab:generation}
    \resizebox{1.0\textwidth}{!}{
    {\setlength{\tabcolsep}{10pt}
    \begin{tabular}{l l l| c c c}
        \toprule
        Methods & Regression positions & Injection positions & WMT14 En--Fr & XSum & GSM8K \\
        \midrule
        Zero-shot & -- & -- & 28.35 {\large ($\pm$0.7)} & 22.89 {\large ($\pm$0.1)} & 22.33 {\large ($\pm$1.3)} \\
        ICL & -- & -- & 35.68 {\large ($\pm$0.7)} & 36.12 {\large ($\pm$0.5)} & 49.67 {\large ($\pm$1.9)} \\
        \midrule
        $W_{0,0}$ (\oursmethod{}) & First position & First position & \textbf{30.46} {\large ($\pm$0.8)} & \textbf{26.01} {\large ($\pm$0.6)} & 20.44 {\large ($\pm$1.0)} \\
        $W_{\mathrm{all},0}$ & All positions & First position & 30.11 {\large ($\pm$1.1)} & 25.81 {\large ($\pm$0.5)} & 20.44 {\large ($\pm$1.0)} \\
        $W_{\mathrm{all},\mathrm{all}}$ & All positions & All positions & 29.90 {\large ($\pm$0.7)} & 25.85 {\large ($\pm$0.3)} & \textbf{28.56} {\large ($\pm$1.2)} \\
        $W_{0,\mathrm{all}}$ & First position & All positions & 17.17 {\large ($\pm$0.6)} & 20.99 {\large ($\pm$0.9)} & 1.56 {\large ($\pm$0.3)} \\
        \bottomrule
    \end{tabular}
    }
    }
\end{table*}

%% file: checklist.tex
\section*{NeurIPS Paper Checklist}

\begin{enumerate}

\item {\bf Claims}
    \item[] Question: Do the main claims made in the abstract and introduction accurately reflect the paper's contributions and scope?
    \item[] Answer: \answerYes{}
    \item[] Justification: The abstract and introduction consistently reflect the paper’s key idea (distributional alignment as a design criterion) and clearly present the proposed metric (\oursmetric{}) and method (\oursmethod{}). The experimental scope (multiple benchmarks and transfer settings) and claims are consistent with the results presented in Section \ref{sec:exp}.

\item {\bf Limitations}
    \item[] Question: Does the paper discuss the limitations of the work performed by the authors?
    \item[] Answer: \answerYes{}
    \item[] Justification: The paper discusses the limitations in Appendix \ref{app:limitation}, including the adoption of a linear mapping and the focus on short-output tasks.

\item {\bf Theory assumptions and proofs}
    \item[] Question: For each theoretical result, does the paper provide the full set of assumptions and a complete (and correct) proof?
    \item[] Answer: \answerYes{}
    \item[] Justification: 
    The assumptions are clearly stated in Proposition \ref{prop:kl_bound}, and a complete proof is provided in Appendix \ref{app:proof_kl_bound}, with all necessary steps to ensure correctness.

    \item {\bf Experimental result reproducibility}
    \item[] Question: Does the paper fully disclose all the information needed to reproduce the main experimental results of the paper to the extent that it affects the main claims and/or conclusions of the paper (regardless of whether the code and data are provided or not)?
    \item[] Answer: \answerYes{}
    \item[] Justification: The paper provides sufficient details to reproduce the experimental results. The datasets (eight classification and two regression benchmarks), models, and experimental setup are described in Section \ref{sec:exp}. Additional implementation details (prompts, baselines, and hyperparameters) are provided in Appendix~\ref{app:exp_details}--\ref{app:transfer}. Code is also provided, enabling reproducibility of the main results.

\item {\bf Open access to data and code}
    \item[] Question: Does the paper provide open access to the data and code, with sufficient instructions to faithfully reproduce the main experimental results, as described in supplemental material?
    \item[] Answer: \answerYes{}
    \item[] Justification: The code is publicly available at \href{https://anonymous.4open.science/r/LTV-9AB6/README.md}{this GitHub repository}, as also referenced in Section \ref{sec:exp}. The experiments use standard benchmark datasets. The repository provides sufficient instructions to reproduce the main results, including task vector extraction and inference for classification and regression tasks, as well as transfer experiments.

\item {\bf Experimental setting/details}
    \item[] Question: Does the paper specify all the training and test details (e.g., data splits, hyperparameters, how they were chosen, type of optimizer) necessary to understand the results?
    \item[] Answer: \answerYes{}
    \item[] Justification: Experimental settings, including number of demonstrations, evaluation setup, run configurations, and data sources are detailed in Section \ref{subsec:exp_setup} and Appendix \ref{app:exp_details}.

\item {\bf Experiment statistical significance}
    \item[] Question: Does the paper report error bars suitably and correctly defined or other appropriate information about the statistical significance of the experiments?
    \item[] Answer: \answerYes{}
    \item[] Justification: Table \ref{tab:comparison_with_baseline} reports results as mean and standard deviation over 5 runs with different random seeds, where the variability arises from the random selection of demonstrations and evaluation. This provides information about the statistical variability of the reported results.

\item {\bf Experiments compute resources}
    \item[] Question: For each experiment, does the paper provide sufficient information on the computer resources (type of compute workers, memory, time of execution) needed to reproduce the experiments?
    \item[] Answer: \answerYes{}
    \item[] Justification: The paper specifies the hardware used (a single NVIDIA RTX 6000 Ada Generation or NVIDIA GeForce RTX 4090) in Appendix \ref{app:exp_details}, and reports extraction and inference time in Table \ref{tab:inference_time}, enabling for estimation of computational latency.
    
\item {\bf Code of ethics}
    \item[] Question: Does the research conducted in the paper conform, in every respect, with the NeurIPS Code of Ethics \url{https://neurips.cc/public/EthicsGuidelines}?
    \item[] Answer: \answerYes{}
    \item[] Justification: The research involves evaluation on standard benchmarks, without human subjects or sensitive data, and thus conforms to the NeurIPS Code of Ethics.

\item {\bf Broader impacts}
    \item[] Question: Does the paper discuss both potential positive societal impacts and negative societal impacts of the work performed?
    \item[] Answer:  \answerNA{}
    \item[] Justification: The paper primarily focuses on a general methodological contribution without a specific application, and therefore does not have identifiable societal impacts.
    
\item {\bf Safeguards}
    \item[] Question: Does the paper describe safeguards that have been put in place for responsible release of data or models that have a high risk for misuse (e.g., pre-trained language models, image generators, or scraped datasets)?
    \item[] Answer: \answerNA{}
    \item[] Justification: The paper proposes a general methodological contribution and uses standard benchmark datasets, which do not pose a high risk for misuse, and therefore no specific safeguards are required.

\item {\bf Licenses for existing assets}
    \item[] Question: Are the creators or original owners of assets (e.g., code, data, models), used in the paper, properly credited and are the license and terms of use explicitly mentioned and properly respected?
    \item[] Answer: \answerYes{}
    \item[] Justification: The paper uses standard benchmark datasets and pretrained models that are widely adopted in the literature. All assets are properly cited and are used under their original licenses or standard research-use terms. Dataset details are summarized in Table~\ref{tab:benchmark_details}.

\item {\bf New assets}
    \item[] Question: Are new assets introduced in the paper well documented and is the documentation provided alongside the assets?
    \item[] Answer: \answerYes{}
    \item[] Justification: The implementation of the proposed method will be released via a public repository. Documentation will be provided together with the code.

\item {\bf Crowdsourcing and research with human subjects}
    \item[] Question: For crowdsourcing experiments and research with human subjects, does the paper include the full text of instructions given to participants and screenshots, if applicable, as well as details about compensation (if any)? 
    \item[] Answer: \answerNA{}
    \item[] Justification: The research does not involve any crowdsourcing experiments or studies involving human participants.

\item {\bf Institutional review board (IRB) approvals or equivalent for research with human subjects}
    \item[] Question: Does the paper describe potential risks incurred by study participants, whether such risks were disclosed to the subjects, and whether Institutional Review Board (IRB) approvals (or an equivalent approval/review based on the requirements of your country or institution) were obtained?
    \item[] Answer: \answerNA{}
    \item[] Justification: The research does not involve research with human subjects, and therefore IRB approval is not applicable.

\item {\bf Declaration of LLM usage}
    \item[] Question: Does the paper describe the usage of LLMs if it is an important, original, or non-standard component of the core methods in this research? Note that if the LLM is used only for writing, editing, or formatting purposes and does \emph{not} impact the core methodology, scientific rigor, or originality of the research, declaration is not required.
    \item[] Answer: \answerNA{}
    \item[] Justification: LLMs are not a core methodological component and are used as standard pretrained models for evaluation.

\end{enumerate}